%% file: main.tex
\documentclass[a4paper,11pt,onecolumn]{article}

\usepackage[margin=2.5cm]{geometry}
\usepackage{authblk} % authors

\usepackage[T1]{fontenc}
\usepackage{csquotes}
\usepackage{amsfonts,amssymb,amsmath,amsthm,bbm}
\usepackage{thmtools,thm-restate}
\usepackage{hyperref}
\usepackage{cleveref}
\usepackage[acronym,nohypertypes={acronym}]{glossaries}
\usepackage{orcidlink}

\usepackage{caption}
\usepackage{subcaption}

\usepackage{makecell}
\usepackage{booktabs}
\usepackage{multirow}
\usepackage{siunitx}
\usepackage{tabularx}
\usepackage{array}

\usepackage{latexsym}
\usepackage{enumitem}
\usepackage{graphicx}
\usepackage{color}

\usepackage{tikz}
\usepackage{pgfplots}
\pgfplotsset{compat=1.18}
\usepackage{float}
\usepackage{rotating}
\usepackage{placeins}

\usepackage[backend=biber,style=numeric-comp,sorting=none,natbib=true]{biblatex}
\newtheorem{theorem}{Theorem}
\newtheorem{lemma}[theorem]{Lemma}

\newtheorem{definition}{Definition}

\newcommand{\1}{\mathbbm{1}}
\newcommand{\R}{\mathbb{R}}

\newcolumntype{?}{!{\vrule width 1pt}}

\newacronym{emm}{EMM}{Exceptional Model Mining}
\newacronym{arl}{ARL}{Average Ranking Loss}
\newacronym{asl}{ASL}{Average Subranking Loss}
\newacronym{rasl}{RASL}{Relative Average Subranking Loss}
\newacronym{scape}{SCaPE}{Soft Classifier Performance Evaluation}
\newacronym{roc}{ROC}{Receiver Operating Characteristic}
\newacronym{pr}{PR}{Precision-Recall}
\newacronym{auc}{AUC}{area under the curve}
\newacronym{iou}{IoU}{Intersection over Union}
\newacronym{fwer}{FWER}{familywise error rate}
\newacronym{fdr}{FDR}{false discovery rate}
\newacronym{nlp}{NLP}{Natural Language Processing}
\newacronym{ncr}{NCR}{Negative Class Ratio}

\date{}

\title{SubROC: AUC-Based Discovery of Exceptional Subgroup Performance for Binary Classifiers}

\author[1]{Tom Siegl \orcidlink{0000-0003-2292-1188}}
\author[1]{Kutalm\i\c s Co\c skun \orcidlink{0000-0001-7680-3182}}
\author[1]{Bjarne C. Hiller \orcidlink{0009-0005-9371-1702}}
\author[1]{Amin Mirzaei \orcidlink{0009-0001-5210-7956}}
\author[2]{Florian Lemmerich \orcidlink{0000-0001-7620-1376}}
\author[1,3]{Martin Becker \orcidlink{0000-0003-4296-3481}}

\affil[1]{University of Rostock, 18055 Rostock, Germany}
\affil[2]{University of Passau, 94032 Passau, Germany}
\affil[3]{Hessian AI / Marburg University, 35037 Marburg, Germany}

\begin{document}

\maketitle

\begin{abstract}
\input{abstract}
\end{abstract}

\section{Introduction} \label{sec:introduction}
\input{introduction}

\section{Preliminaries} \label{sec:preliminaries}
\input{preliminaries}

\section{The SubROC Framework} \label{sec:methods}
\input{method}

\section{Experiments} \label{sec:experiments}
\input{experiments}

\section{Related Work} \label{sec:related_work}
\input{related_work}

\section{Conclusion}
\input{conclusion}

\section{Acknowledgements}
Funding was provided by the BMFTR (01IS22077, 01ED2507).

\printbibliography

\newpage

\appendix
\input{supplementary_material}

\end{document}

%% file: abstract.tex
Machine learning (ML) is increasingly employed in real-world applications like medicine or economics, thus, potentially affecting large populations.
However, ML models often do not perform homogeneously, leading to underperformance or, conversely, unusually high performance in certain subgroups (e.g., sex=female $\wedge$ marital\_status=married).
Identifying such subgroups can support practical decisions on which subpopulation a model is safe to deploy or where more training data is required.
However, an efficient and coherent framework for effective search is missing.
Consequently, we introduce \emph{SubROC}, an open-source, easy-to-use framework based on Exceptional Model Mining for reliably and efficiently finding strengths and weaknesses of classification models in the form of interpretable population subgroups.
SubROC incorporates common evaluation measures (ROC and PR AUC), efficient search space pruning for fast exhaustive subgroup search, control for class imbalance, adjustment for redundant patterns, and significance testing.
We illustrate the practical benefits of SubROC in case studies as well as in comparative analyses across multiple datasets.

%% file: introduction.tex
Machine learning (ML) is widely adopted in high-stake domains such as health care, loans, hiring, or legal decision-making \cite{mehrabi_survey_2021}. 
Thus, ML model predictions already broadly influence everyday life.
However, it has also repeatedly been shown that ML models can be highly heterogeneous in their performance \cite{mehrabi_survey_2021}.
Subgroups of a population may exist where a ML model underperforms (or performs remarkably well).
For illustration, see \Cref{fig:figure-1} where 
 predictions are particularly bad for women from rural areas (sex=female $\wedge$ housing=rural) and remarkably accurate for male teenagers (sex=male $\wedge$ age$\leq$18).
Knowing such subgroups may allow to make informed decisions about model deployment and development.
Here, practitioners may decide to use the current model only for male teenagers where the model performs well, 
and data scientists may decide to collect more data 
for subgroups where the model underperforms.

\begin{figure}[ht]
    \centering
    \includegraphics[width=\linewidth]{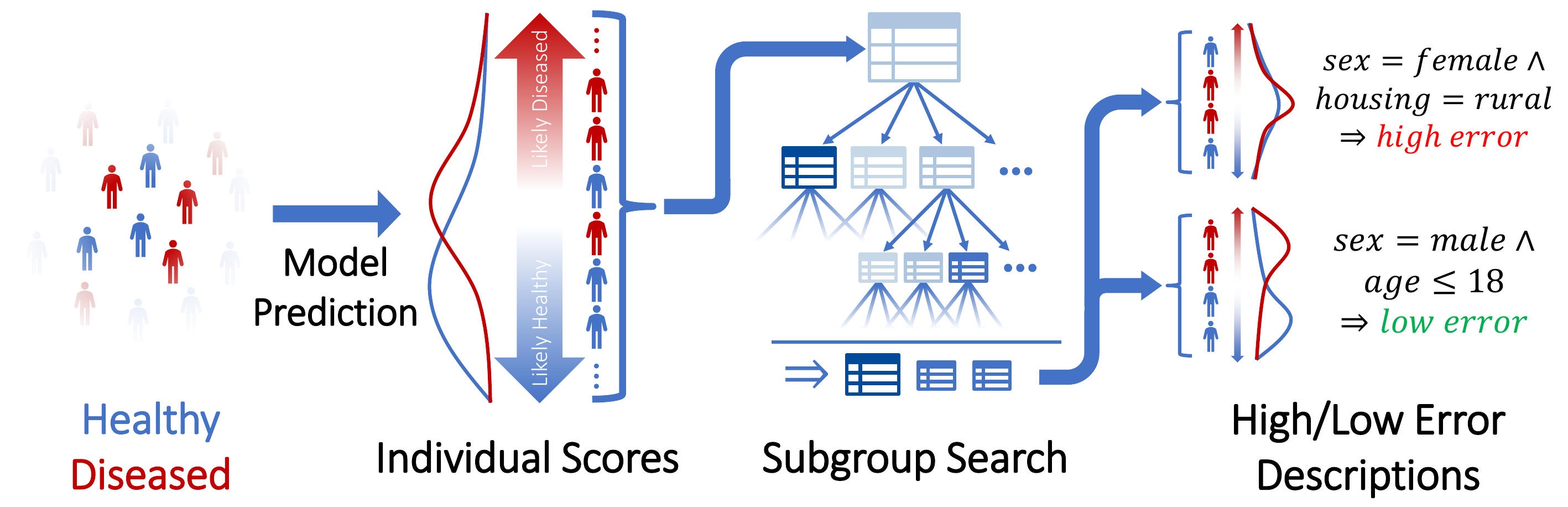}
    \caption{
        \textbf{Illustrative example.} 
        Predictions of a binary soft classifier produce a ranking of patients, from healthy to sick. 
        The clearer this separation, the better the classifier. 
        Here, the density distributions of healthy and diseased patients (red and blue curves, middle left) are clearly distinguishable but there is no perfect separation.
        SubROC automatically discovers interpretable subgroups where the model predictions create the best or worst separation of the classes (middle right).
        Results may highlight arbitrary sensitive subgroups of a population, where the model under- or over-performs compared to the rest of the dataset (right).
    }
    \label{fig:figure-1}
\end{figure}

\textbf{Background and problem setting.}
The need and potential of identifying and analyzing subgroups has been widely recognized.
For example, a wide array of work aims at alleviating consistent underperformance in known \cite{sagawa_distributionally_2020} as well as previously unknown \cite{sohoni_no_2020} subgroups.
Similar observations have been made 
concerning the fairness of ML models \cite{mehrabi_survey_2021}.
Consequently, various approaches explore such atypical subgroups systematically, e.g.,
via interactive tools for manual discovery \cite{baylor_tfx_2017,kahng_visual_2016}.
However, this is labor-intensive, inefficient, and poses the threat of missing essential subgroups.
Thus, automated approaches enable the search for subgroups where ML models over- or underperform \cite{chung_automated_2020,cabrera_fairvis_2019}.
The most principled approach employs \gls{emm} with a custom performance measure for efficient exhaustive search \cite{duivesteijn_understanding_2014}.
However, none of these approaches combines the automated and exhaustive search of \gls{emm}, standard measures for soft classifiers like ROC AUC or PR AUC, efficient search, and handling of class imbalance and statistical significance testing, which are required for a routine and practical analysis.

\textbf{Contribution.}
Consequently, we introduce \emph{SubROC}, which builds on \gls{emm} to reliably and efficiently find strengths and weaknesses of classification models:
We
(a) propose a new group of scoring functions based on the area-under-the-curve (AUC) statistic,
(b) derive tight optimistic estimates enabling fast exhaustive subgroup search,
(c) provide a well-rounded framework that includes control for class imbalance, pruning of redundant results, and significance testing,
(d) show the practical benefits of our approach across multiple datasets, and
(e) integrate SubROC in the easy-to-use, open source \emph{pysubgroup} package and provide the code for our experiments. \footnote{\url{https://github.com/bckrlab/subroc}, archival version at \cite{siegl_code_2025}}

SubROC will enable practitioners to routinely evaluate ML models by finding subgroups where they over- or underperform. 
This will not only improve quality control procedures and increase the confidence in deployed ML models but also lay the foundation for downstream approaches to systematically improve their performance.

%% file: preliminaries.tex
 \subsection{Notation}
$D = (\mathcal{I}, \mathcal{A})$ is a dataset
with instances $c \in \mathcal{I}$ 
and attributes (features) $\mathcal{A} = (X_1, \ldots, X_m)$.
$X_i(c)$ denotes the value of attribute $X_i$ in instance $c$.
$Y(c) \in \{1,0\}$ is the ground truth label 
and
$\hat{Y}(c) \in [0,1]$ a continuous (soft) prediction. 
We call instances with $Y(c) = 1$ positives ($\mathcal{P}_{\mathcal{I}} = \{c \in \mathcal{I} | Y(c) = 1\}$) and the rest negatives ($\mathcal{N}_{\mathcal{I}} = \{c \in \mathcal{I} | Y(c) = 0\}$).
A classifier performs well if it returns high (low) values for positives $\mathcal{P}_{\mathcal{I}}$ (negatives $\mathcal{N}_{\mathcal{I}}$).

\subsection{Soft Classifier Evaluation}
Traditional subgroup discovery on $Y \neq \hat{Y}$ is applicable to hard classification but ignores classifier confidence.
Since in practice most models provide probability estimates, SubROC focuses on soft classification with the performance measures \gls{roc} \gls{auc} and \gls{pr} \gls{auc} that require labels $Y$ as well as prediction scores $\hat{Y}$.
This leads to conceptually different, not necessarily better results.

While many evaluation measures for soft classification exist, including binary cross-entropy or the Brier score, the \gls{auc} of the \gls{roc} curve is among the most widely used.
It assumes a ranking of instances according to their predicted probabilities from high to low.
It can be interpreted as the probability of a randomly chosen positive instance being ranked above a randomly chosen negative instance.

One reason for the popularity of AUC-based performance measures is that they summarize classifier performance across all potential thresholds that a user can apply on soft classifier probabilities to separate samples into two classes 
e.g., depending on misclassification costs.
This makes the ROC AUC a natural choice for interpretable subgroup descriptions, contributing to the insights 
our framework may provide to experts.
Nevertheless, ROC AUC has also been criticized, for example, for being misleading under strong class imbalance and for treating all thresholds as equally relevant.

\subsection{Subgroup Discovery / Exceptional Model Mining}\label{sec:subgroup_discovery}
A \emph{selector} $s \in \mathcal{S}$ is a Boolean function based on a single attribute, e.g., $sex=female$ or $age \in [20, 30)$. 
A \emph{pattern} $p \in 2^\mathcal{S}$ is a conjunction of selectors, e.g., $sex=female \wedge age \in [20, 30)$.
Given a dataset, a pattern's \emph{cover} $sg(p)$ represents the set of all instances \emph{covered} by it, i.e., all instances with $p=true$.
A pattern and its cover are referred to as a \emph{subgroup}.
A subgroup $p_{spec}$ is a \emph{specialization} of another subgroup $p_{gen}$ (its \emph{generalization}) if its set of selectors is a proper superset of those of the other subgroup: $p_{gen} \subset p_{spec}$.

Subgroup discovery (SD) \cite{klosgen_explora_1996} finds subgroups that show an interesting deviation with respect to a \emph{target concept}, which is traditionally given as a single data attribute. 
E.g., an interesting subgroup 
may contain a higher-than-average disease occurrence.
\gls{emm} \cite{leman_exceptional_2008} extends target concepts to \emph{model classes}. 
Here, a subgroup is considered interesting if the model parameters fitted to the subgroup are significantly different from those fitted to the complete dataset.

SD and \gls{emm} formalize interestingness using \emph{scoring functions} $\varphi: 2^{\mathcal{S}} \rightarrow \mathbb{R}$ (also called \emph{interestingness measures} or \emph{quality measures}) which map a pattern $p \in 2^{\mathcal{S}}$ to a score, given a dataset $D$.
A search algorithm explores the \emph{search space} of patterns ($2^\mathcal{S}$) to find the ones with the highest scores.
The computational challenge of subgroup discovery lies in the exponentially growing
search space $2^\mathcal{S}$ with the number of selectors $|\mathcal{S}|$.
Beyond straightforward constraints like the number of selectors in a pattern or the minimum support, in terms of the number of covered instances,
\emph{optimistic estimates} are integral for optimizing runtimes.
An optimistic estimate on a score ${\varphi}$,
$oe_{\varphi}$ 
provides 
an upper bound on the score $\varphi$ of all specializations $p_{spec}$ of a 
subgroup $p_{gen}$ \cite{lemmerich_novel_2014}:
 $\forall p_{spec} \supset p_{gen}: \varphi(p_{spec}) \leq oe_{\varphi}(p_{gen})$.
When searching for the top-k patterns with the highest interestingness, this can be used to prune all specializations of $p_{gen}$ if $oe_{\varphi}(p_{gen})$ is lower than the score of all current top-k patterns.
Furthermore, \emph{tight optimistic estimates} \cite{grosskreutz_tight_2008}
provide tighter upper bounds by guaranteeing that the upper bound can  be reached by a hypothetical specialization $p_{spec}$, for an arbitrary subset of instances $I \subseteq sg(p_{gen})$ covered by $p_{gen}$.
Optimistic estimates can be integrated into various search algorithms. However, they have to be derived for each scoring function individually.

\subsection{SCaPE Model Class} \label{sec:scape}
For \gls{emm}, the \gls{scape} model class has been proposed to discover subgroups with exceptional performance of soft classifiers \cite{duivesteijn_understanding_2014}.
It integrates 
the ground truth labels $Y$ 
and the model predictions $\hat{Y}$.

The original scoring function \gls{rasl} $\varphi^{rasl}$ proposed for \gls{scape}, measures the number of violations of the ranking induced by the predictions $\hat{Y}(c)$ of the given soft classifier, against the ranking based on the ground truth labels $Y(c)$:
For each positive instance $c \in \mathcal{P}_p$ for a pattern $p$,
it counts how many negative instances $c' \in \mathcal{N}_p$
are ranked higher.
This count is denoted as $PEN(c,p)$.
The \gls{arl} is then defined as the average of $PEN(c,p)$ across all positive instances.
The final scoring function $\varphi^{rasl}$ is defined relative to the overall dataset: $\varphi^{rasl}(p) := ARL(p) - ARL(\mathcal{\emptyset})$, where the empty pattern $\emptyset$ corresponds to the subgroup containing all instances.

Although the \gls{arl} score is normalized by the number of positive instances, 
it is heavily influenced by the number of overall instances (of either class) covered by a pattern as the penalty score $PEN(c,p)$ counts the \emph{absolute} number of violations.
This favors large subgroups or heavily biased class 
distributions without an option to manually trade off these aspects (see \Cref{subsec:skew_analysis} for details).

In practice, ARL is not a common score to measure classifier performance and the corresponding RASL measure \cite{duivesteijn_understanding_2014} does not provide optimistic estimates for efficient subgroup search.
ROC and PR AUC, which are widely applied to evaluate soft classifiers, currently are not available for the SCaPE model class.
However, due to their widespread adoption, they are more readily interpretable and have technical advantages like their independence of the number of evaluated samples (also see \Cref{subsec:skew_analysis}).
Thus, SubROC introduces ROC and PR AUC for the SCaPE model class and provides optimistic estimates for ROC and PR AUC as well as RASL.

%% file: method.tex
The SubROC framework consists of three components: 
defining novel scoring functions based on \gls{roc} \gls{auc}, while accounting for class imbalance, size, and redundancy (\Cref{sec:interestingness_measures}), 
deriving optimistic estimates for efficient search space exploration and pruning (\Cref{subsec:optimistic_estimates}),
and significance filtering for increasing the robustness and generalizability of the discovered subgroups (\Cref{sec:significance_filtering}).
The same construction is also applied to \gls{pr} \gls{auc} by replacing \gls{roc}’s true and false positive rates with precision and recall.

\subsection{Scoring Functions} \label{sec:interestingness_measures}

\subsubsection{AUC-based subgroup interestingness}
For a set of instances $I$, ROCAUC and PRAUC both map soft predictions $\hat{Y}(I)$ and corresponding ground truth labels $Y(I)$ to a performance metric, denoted as $ROCAUC(I)$ and $PRAUC(I)$, respectively.
We define the corresponding relative AUC-based subgroup scores given a dataset $D = (\mathcal{I}, \mathcal{A})$ and pattern $p$ as:
\begin{align*}
    \varphi^{rROCAUC}(p) = & ROCAUC(\mathcal{I}) - ROCAUC(sg(p))  \label{eq:def_varphi_rrocauc}\\
    \varphi^{rPRAUC}(p) = & PRAUC(\mathcal{I}) - PRAUC(sg(p))
\end{align*}

To guarantee well-defined scoring,
we define constraints:
at least one positive instance for \gls{arl} and \gls{pr} \gls{auc}
and at least one instance of both classes for \gls{roc} \gls{auc}.
For \gls{pr} \gls{auc}, we employ the widely used approximation via a linearly interpolated \gls{pr} curve, with no practical consequences in our initial tests, cf.~\cite{davis_relationship_2006}.

Subgroup scoring functions commonly include a factor for trading off size against the score to prevent the discovery of only small, overfitted subgroups (e.g., see \cite{klosgen_explora_1996,leman_exceptional_2008}).
We apply the same concept to our AUC-based scoring functions $\varphi'$ and derive size-sensitive scoring functions $\varphi_{\alpha}$ as $\varphi_{\alpha}(p) = |sg(p)|^{\alpha} \cdot \varphi'(p)$.
Here, $\alpha \in \R_{\geq 0}$ is the trade-off parameter, where $\alpha = 0$ removes the size influence.

\subsubsection{Class Balance Factor} \label{sec:class_balance_factor}

The performance measures \gls{roc} \gls{auc} and \gls{pr} \gls{auc} are known to be sensitive to the \gls{ncr} \cite{cook_when_2020,boyd_unachievable_2012}.
This can lead to subgroups with high scores just due to high class imbalances rather than exceptional model performance.
To counterbalance this, we introduce a user-weighted class balance factor $cb(I)$ for our scoring functions \cite{canbek_binary_2017,datta_multiobjective_2019}.
Given a set of instances $I$, it provides a maximum value close to 1 for perfectly balanced classes
and lower values for more imbalanced data.
We define $cb(I)$ as $0$ if only one class occurs and as $\min \{\frac{|\mathcal{N}_I|}{|\mathcal{P}_I|}, \frac{|\mathcal{P}_I|}{|\mathcal{N}_I|}\}$ otherwise. 
We combine $cb(I)$ with 
a scoring function $\varphi'$ as an additional factor, forming a \emph{class balance weighted scoring function} $\varphi_{\beta}(p)$ for a pattern $p$: $\varphi_{\beta}(p) = cb(sg(p))^{\beta} \cdot \varphi'(p)$
with $\beta \in \R_{\geq 0}$.
Scoring functions that include both the cover size weight with parameter $\alpha$ and class balance weight with parameter $\beta$ are denoted as $\varphi_{\alpha, \beta}$.
This follows the intuition that ROC AUC can be deceptive for imbalanced data, and thus, balanced subgroups should be preferred \cite{saito2015precision}.
Alternatives such as the difference in class-ratios between the subgroup and the dataset could also be considered.

\subsubsection{Generalization Awareness} \label{sec:generalization_awareness}

It is a common issue in subgroup discovery that multiple subgroups in the result set have a high pairwise overlap in terms of their cover \cite{lemmerich_difference-based_2013}.
To reduce redundancy along the generalization-specialization hierarchy, adaptations to scoring functions have been proposed under the term \emph{generalization awareness}. 
In this approach, a variant \cite{lemmerich_difference-based_2013} 
of previous approaches \cite{batal_concise_2010,lemmerich_local_2011,grosskreutz_subgroup_2010}
extends a scoring function $\varphi'$ by reducing the score $\varphi'(p)$ of a subgroup with pattern $p$ by subtracting the maximum score across all its generalizations $h$ (making a specialization viable only if its score is substantially higher):
$$\varphi_{\alpha}(p) = |sg(p)|^{\alpha} \cdot (\varphi'(p) - \max_{h \subset p} \varphi'(h)) \text{, with } \alpha \in [0,1].$$
We employ a variation of this which allows applying it to cover size and class balance weighted scoring functions $\varphi_{\alpha, \beta}'$:
$$\varphi_{\alpha, \beta}(p) = \varphi_{\alpha, \beta}'(p) - \max_{h \subset p} \varphi_{\alpha, \beta}'(h)$$ 

\subsection{Optimistic Estimates} \label{subsec:optimistic_estimates}

For runtime optimization, 
we derive tight optimistic estimates for the scoring functions $\varphi^{rasl}$, $\varphi^{rROCAUC}$ and $\varphi^{rPRAUC}$ as well as, separately, for the cover size and class balance factors.
Proofs are given in \Cref{apx:proofs_of_tight_optimistic_estimates}.

\subsubsection{Average Ranking Loss} \label{subsubsec:arl_optimistic_estimate}

To date, no (tight) optimistic estimate was derived for the \gls{arl}-based scoring function $\varphi^{rasl}$ \cite{duivesteijn_understanding_2014}. 
$ARL(p)$ is defined as the average over the ranking losses $PEN(c,p)$ of all covered positive instances $c \in \mathcal{P}_p$. 
The penalty $PEN(c,p)$ for each positive instance $c$ counts the negative instances ranked higher than $c$.
Thus, removing positive instances with non-maximal $PEN(c,p)$ does not alter the $PEN$ of the remaining instances and increases the average by eliminating small values.
The $ARL(p)$ is therefore maximized when only a single positive $c'$ instance with maximal $PEN(c',p)$ is left.

\input{theorem_tight_optimistic_estimate_ARL}

\subsubsection{Area Under the Receiver-Operating Curve} \label{subsubsec:rocauc_optimistic_estimate}

Next, we introduce a lower bound $b$
for \gls{roc} \gls{auc}
that we will use to derive an optimistic estimate of $\varphi^{rROCAUC}$.
We distinguish three cases:
(a) If perfect separation of ground truth labels is possible with a threshold on the soft predictions, the same holds for every subset of the cover. 
Thus, the \gls{roc} \gls{auc} is always 1.
(b)
If the input contains some instances that share the same prediction value (ties) but have different labels (positive and negative), then
the \gls{roc} \gls{auc} on only these tied instances equals $\frac{1}{2}$.
Any subset of the given instances containing both classes and at least one non-tied instance allows for the correct separation of some positives and negatives by at least one threshold.
Thus, the \gls{roc} \gls{auc} is greater than $\frac{1}{2}$ for such subsets.
(c) If the input contains at least one incorrectly ranked pair of a positive and a negative instance.
The subset containing exactly one such pair has a \gls{roc} \gls{auc} of $0$, which is the lowest possible \gls{roc} \gls{auc}.

\input{theorem_lower_bound_ROCAUC}

\input{theorem_tight_optimistic_estimate_ROCAUC}

\subsubsection{Area Under the Precision-Recall Curve} \label{subsubsec:prauc_optimistic_estimate}

\Cref{lemma:tight_bound_prauc} introduces a lower bound for \gls{pr} \gls{auc} which is used in \Cref{theorem:prauc_tight_optimistic_estimate} to derive a tight optimistic estimate for $\varphi^{rPRAUC}$.
For \Cref{lemma:tight_bound_prauc}, a subset $I_{worst}$ of the original set of instances $I$ is constructed
to maximize classification errors while minimizing the positive class ratio.
This makes it impossible for a threshold on any other subset 
to define a point in PR space below the PR curve of $I_{worst}$ (with some exceptions that have no impact on the \gls{auc}).
Proofs and an illustration are included in \Cref{sec:prauc_optimistic_estimate_proofs}.

\input{theorem_lower_bound_PRAUC}

\input{theorem_tight_optimistic_estimate_PRAUC}

\subsubsection{Cover Size and Class Balance Weighted Scoring}

The following \Cref{lemma:tight_bound_weighting} gives a tight upper bound for the weighting used in cover size and class balance weighted scoring functions 
when $\alpha \leq\beta$.
This result translates to optimistic estimates for $\varphi_{\alpha, \beta}^{rasl}(\cdot)$, $\varphi_{\alpha, \beta}^{rROCAUC}(\cdot)$ and $\varphi_{\alpha, \beta}^{rPRAUC}(\cdot)$.
It therefore enables the use of the optimistic estimates presented in theorems \ref{theorem:arl_tight_optimistic_estimate}, \ref{theorem:rocauc_tight_optimistic_estimate} and \ref{theorem:prauc_tight_optimistic_estimate} when cover size and class balance weighting is enabled.

\input{theorem_upper_bound_weighting}

The product $b_w(sg(\cdot)) \cdot \varphi'$ with $\varphi' \in \{oe_{\varphi^{rasl}}(\cdot), oe_{\varphi^{rROCAUC}}(\cdot), oe_{\varphi^{rPRAUC}}(\cdot)\}$ is an optimistic estimate of the cover size and class balance weighted scoring functions $\varphi_{\alpha, \beta}^{rasl}(\cdot)$, $\varphi_{\alpha, \beta}^{rROCAUC}(\cdot)$ and $\varphi_{\alpha, \beta}^{rPRAUC}(\cdot)$.

\subsection{Significance Filtering} \label{sec:significance_filtering}

Subgroup discovery inspects an exponential search space of patterns and consequently requires correction for multiple hypothesis comparison \cite{chung_automated_2020}.
To reduce the number of compared subgroups, we use a holdout procedure \cite{webb_discovering_2007}: We split the dataset into two parts, the \emph{search dataset} to search subgroups and the \emph{validation dataset} to test statistical significance.
The tested statistic is the raw score of a subgroup (written $\varphi_{0, 0}$) excluding cover size and class imbalance. 
The null hypothesis is that $\varphi_{0, 0}(p)$ for a pattern $p$ is drawn from the sample distribution of $\varphi_{0, 0}$ on random instances with the same size and class balance as the cover $sg(p)$ of $p$.
This allows a test on performance differences, excluding the influence of cover size and class balance on the used performance measure.

We compute the \emph{empirical p-value} of a subgroup $p$ by a randomization test (cf. \cite{gionis_assessing_2007}), where $1000$ random subsets of the validation dataset are drawn without replacement.
To prevent significance due to size and class balance, we keep these parameters constant during sampling.
The empirical p-value is the fraction of random subsets with interestingness $\varphi_{0, 0}$ greater than or equal to $\varphi_{0, 0}(p)$ divided by the total number of random subsets.

\emph{Multiple testing correction} is done
using Benjamini-Yekutieli (BY) \cite{benjamini_control_2001} (controlling the \gls{fdr}).
BY has the necessary properties to be applied under any dependency between tested hypotheses.
The correction approach can be customized by the user, e.g., with Bonferroni controlling \gls{fwer}.
The default p-value threshold we use is $p \leq \alpha$. Here $\alpha = 0.05$.

This enables \emph{significance filtering}, focusing the search on significant subgroups:
To find top-$k$ significant subgroups (e.g., $k=5$), search for the top-$k'$ subgroups with $k' \gg k$ (e.g., $k' = 100$), calculate p-values, and only include a subgroup in the final top-5 if their corrected p-value passes the significance threshold $\alpha$.

%% file: theorem_tight_optimistic_estimate_ARL.tex
\begin{restatable}[Tight Optimistic Estimate for $\varphi^{rasl}$]{theorem}{tightoearl}
    \label{theorem:arl_tight_optimistic_estimate}
    Given a dataset $D = (\mathcal{I}, \mathcal{A})$ and a pattern $p$, a tight optimistic estimate for $\varphi^{rasl}(p)$ is
    $oe_{\varphi^{rasl}}(p) = \max_{c \in \mathcal{P}_p} \{PEN(c,p)\} - ARL(\emptyset)$,
    with $\mathcal{P}_p = \{c \in sg(p) | Y(c) = 1\}$.
\end{restatable}

%% file: theorem_lower_bound_ROCAUC.tex
\begin{restatable}[Tight Lower Bound $b_{ROCAUC}$]{lemma}{tightboundrocauc}
    \label{lemma:tight_bound_rocauc}
    For a set of instances $I$ containing at least one positive and one negative instance, a tight lower bound for the ROC AUC of all subsets of $I$ is given by
    \begin{equation*}
        b_{ROCAUC}(I) =
        \begin{cases}
            1,& \begin{aligned}
            \text{if } \forall c, c' \in I: & (y < y' \rightarrow \hat{y} < \hat{y}') \wedge\\
            & (y > y' \rightarrow \hat{y} > \hat{y}')
            \end{aligned}\\
            \frac{1}{2},& \begin{aligned}
            \text{if } \forall c, c' \in I: & (y < y' \rightarrow \hat{y} \leq \hat{y}') \wedge\\
            & (y > y' \rightarrow \hat{y} \geq \hat{y}')
            \end{aligned}\\
            0,& \text{otherwise}
        \end{cases}
    \end{equation*} 
    with $y$ and $y'$ ($\hat{y}$ and $\hat{y'}$) denoting $Y(c)$ and $Y(c')$ ($\hat{Y}(c)$ and $\hat{Y}(c')$), respectively.
\end{restatable}

%% file: theorem_tight_optimistic_estimate_ROCAUC.tex
\begin{restatable}[Tight Optimistic Estimate  for $\varphi^{rROCAUC}$]{theorem}{tightoerocauc}
    \label{theorem:rocauc_tight_optimistic_estimate}
    Given a dataset $D = (\mathcal{I}, \mathcal{A})$ and a pattern $p$,
    a tight optimistic estimate for $\varphi^{rROCAUC}(p)$ is
    $oe_{\varphi^{rROCAUC}}(p) = ROCAUC(\mathcal{I}) - b_{ROCAUC}(sg(p))$.
\end{restatable}

%% file: theorem_lower_bound_PRAUC.tex
\begin{restatable}[Tight Lower Bound $b_{PRAUC}$]{lemma}{tightboundprauc}
    \label{lemma:tight_bound_prauc}
    For a set of instances $I$ containing at least one positive instance, a tight lower bound for the \gls{pr} \gls{auc} of all subsets of $I$ is given by
    $b_{PRAUC}(I) = PRAUC(I_{worst})$,
    where $I_{worst} = \{\text{argmin}_{c \in \mathcal{P}_I} \{\hat{Y}(c)\}\} \cup \mathcal{N}_I$.
\end{restatable}

%% file: theorem_tight_optimistic_estimate_PRAUC.tex
\begin{restatable}[Tight Optimistic Estimate for $\varphi^{rPRAUC}$]{theorem}{tightoeprauc}
    \label{theorem:prauc_tight_optimistic_estimate}
    Given a dataset $D = (\mathcal{I}, \mathcal{A})$ and a pattern $p$, a tight optimistic estimate for $\varphi^{rPRAUC}(p)$ is $oe_{\varphi^{rPRAUC}}(p) = PRAUC(\mathcal{I}) - b_{PRAUC}(sg(p))$.
\end{restatable}

%% file: theorem_upper_bound_weighting.tex
\begin{restatable}[Weighting Tight Upper Bound]{lemma}{tightboundweighting}
    \label{lemma:tight_bound_weighting}
    For a set of instances $I$, a tight upper bound of the overall weighting $w(I) = |I|^{\alpha} \cdot cb(I)^{\beta}$ with $\alpha \leq \beta \in \R_{> 0}$ is given by
    $b_{w}(I) = (2 \cdot \min\{|\mathcal{P}_I|, |\mathcal{N}_I|\})^{\alpha}$.
\end{restatable}

%% file: experiments.tex
Datasets are from OpenML \cite{vanschoren_openml_2013} and the UCI Machine Learning Repository \cite{kelly_uci_nodate} (statistics in \Cref{tab:datasets_overview}).
Each dataset was partitioned randomly into three equal parts:
(i) a training set for fitting the prediction model;
(ii) a search set where the trained model is applied to yield predictions $\hat{Y}$ on unseen data and the subgroup discovery is performed;
(iii) a validation set for significance testing of the found subgroups (see \Cref{sec:significance_filtering}).

\subsection{Quality of SubROC Across Multiple Datasets}
\label{subsec:exp_quality_parameters}

\begin{table*}[ht]
    \centering
    \caption{
    Overview of top-5 results from a basic SubROC setting (\texttt{b}) and the full SubROC framework (f) for a range of datasets and performance measures.
    Full SubROC managed to find diverse and exceptional subgroups, which are often much larger and balanced than the baseline.
    Numbers are bold for the better result when comparing \texttt{b} and \texttt{f}.
    (Cont. in \Cref{tab:generalizability_metrics_overview}, containing results for more datasets)
    }
    \input{table_result_quality_summary_depth_4_manually_rearranged_main_text_merged}
    \label{tab:generalizability_metrics_overview_part_1}
\end{table*}

\begin{table}[ht]
    \centering
    \caption{Example filtered-5 result set from the full SubROC framework using ROC AUC on \texttt{Adult} data with \texttt{xgboost} predictions.}
    \input{table_case_study_filtered_result_set_table_1_1_sklearn.metrics.roc_auc_score_True_reduced}
    \label{tab:example_result_set}
\end{table}

To illustrate the effectiveness of the full SubROC framework, we compare a basic SubROC setting (\texttt{b}) and the full SubROC framework (\texttt{f}) for finding the top-5 underperforming subgroups (\Cref{tab:generalizability_metrics_overview_part_1}).
Both cases train \texttt{xgboost} \cite{chen2016xgboost} as the binary classifier on (i), chosen because it is a widely adopted, strong baseline method.
Subgroups are then exhaustively\footnote{Results are the same for any non-heuristic (exhaustive) search algorithm.} searched on (ii) with a minimum subgroup size of 20, and a maximum pattern length of 4.
The basic setting (\texttt{b})
solely uses the bare scoring functions $\varphi^{rasl}$, $\varphi^{rROCAUC}$, and $\varphi^{rPRAUC}$
and returns the top-5 subgroups.
Their significance is tested on (iii) to underpin their tendency for random artifacts (e.g. very small subgroups, high class imbalance), see column \emph{Significant / Top-5}.
The full framework (\texttt{f}) is configured 
with size and class imbalance weighting $\alpha = \beta = 1$, 
generalization awareness, 
and statistical significance filtering ($k' = 100$).
To illustrate the effectiveness of significance filtering, 
we report how many subgroups were found for a top-100 search on (ii),
how many of these passed the statistical test on (iii), 
and how many finally remain after significance filtering (see column \emph{Filtered-5 / Significant / Top-100}).
Statistics were calculated for the final top-5 results on (ii) in both cases.
Note that exceptionality should not be compared between different performance measures because of their differing scales and responses to errors.

For \texttt{b}, 
cover sizes tend to be close to the minimum cover size limit, and class ratios tend to be close to 0 or 1.
Despite their high difference in performance compared to the overall dataset (exceptionality), such subgroups do not necessarily point towards structural issues of the classifier but may be caused by artifacts (cf. \Cref{subsec:skew_analysis}).
This is also apparent in the low number of subgroups passing statistical significance in the top-5 results of \texttt{b}.
Similar patterns are observable in the top-100 results (see \Cref{tab:generalizability_metrics_overview_filtered}).
While yielding less prominent exceptionality, we observe that the SubROC framework leads to larger and more balanced results.
Additionally, returned subgroups are substantially more likely to pass the significance filtering step. 

Mean pairwise \gls{iou} of the discovered subgroups was computed to assess redundancy.
Values for SubROC remain comparable to the baseline, with some significant improvements, indicating that our framework does not sacrifice diversity.

The significance filtering step of SubROC saves resources for practitioners by discarding a considerable amount of initially found, but unreliable subgroups.
This effect is particularly prominent for small datasets with less than $2,500$ instances in the search set, where the significance filtering discarded almost all subgroups.
It is less prominent for \gls{arl} based scoring functions, which generally produce smaller result sets with a greater share of significant subgroups.

\Cref{tab:example_result_set} shows one of the result sets from this experiment.
For more full result sets on the \texttt{Adult} dataset, see \Cref{apx:sec:case_study_adult}.

\subsection{Runtime Speed Ups with Optimistic Estimates}
\label{subsec:exp_optimistic_estimates}

\begin{table}[t]
    \centering
    \caption{
        Speedups of median runtimes using optimistic estimates, given in terms of median runtimes without optimistic estimate pruning divided by median runtimes with optimistic estimate pruning.
        (Cont. in \Cref{tab:time_speedups_appendix}, containing results for more datasets)
    }
    \input{table_optimistic_estimates_xgboost_depth_4_time_median_speedup_table_main_split}
    \label{tab:time_speedups}
\end{table}

To investigate the efficiency improvements enabled by SubROC, we measure runtimes of subgroup discovery with and without pruning based on our optimistic estimates, carrying out 10 repetitions for each dataset (3 repetitions for Census due to long runtimes).
The searches for the top-5 patterns run on predictions from an \texttt{xgboost} classifier, using the best-first-search enumeration strategy, with a maximum pattern length of 4.
Different size and class balance weightings ($\alpha, \beta$) are compared.
\Cref{tab:time_speedups} gives an overview of the measured speedups.
Absolute runtimes are given in \Cref{apx:sec:optimistic_estimates}.

The results show substantial runtime reductions.
Particularly large speedups were achieved on the \texttt{Bank} dataset.
For \gls{arl}, pruning cut the runtime from $\sim 2.13$ min to $<0.3$ s ($\sim 482 \times$ speedup).
For \gls{pr} \gls{auc}, the runtime dropped from $\sim 2.42$ min to $0.3$ s ($\sim 461 \times$ speedup).
For \gls{roc} \gls{auc}, pruning reduced the runtime from $\sim 2.67$ min to $0.4$ s ($\sim 402 \times$ speedup).
With regard to parameterization, both ignoring size and class balance altogether ($\alpha = \beta = 0$) and weighting them equally to the performance measure ($\alpha = \beta = 1$) appear to produce the highest speedups.
Further investigations are required, however, an explanation for this may be that the search can more strongly focus on one or the other optimistic estimate factor.

\subsection{Discovering an Injected Subgroup}
\label{subsec:exp_subgroup_injection}

\input{figure_injection_experiment_combined_roc_auc}

To investigate the validity of SubROC and its ability to discover expected subgroups, we injected an artificially underperforming subgroup into the \texttt{Adult} dataset based on the predictions of an \texttt{xgboost} classifier.
The subgroup was randomly drawn from candidates with pattern length $\leq 3$ whose cover size fell between $0.4\%$–$0.6\%$ of the search set.
It covers 54 instances, 
has an \gls{ncr} of 
0.2, 
and an \gls{arl} of 
1.07,
\gls{pr} \gls{auc} of
0.97, and
\gls{roc} \gls{auc} of
0.90
before and 
9.93,
0.62,
and 
0.1, 
respectively,
after injection.
For the injection, we multiplied the predicted scores in the subgroup by $-1$.
Thereby, the injected subgroup forms a known weakness in the predictions, and we can evaluate the subgroup search by whether this subgroup is included in the result set.
The subgroup discovery was performed up to depth 4 under different configurations (\emph{baseline} without the SubROC framework and three different weightings of cover size and class balance in the SubROC framework).
Result sets were top-10 filtered to reflect interpretable result set sizes.

\Cref{fig:subgroup_injection_result_roc_auc} shows the \gls{iou} as an overlap measure between found subgroups and the injected subgroup for the \gls{roc} \gls{auc}.
In contrast to the baseline, the SubROC framework successfully discovers the injected subgroup and reports it at the highest ranks in the result set.
Note that multiple patterns can describe the same set of instances, so the injected errors can be found several times.

While \gls{roc} \gls{auc} robustly identifies the subgroup, similar experiments on \gls{arl} or the \gls{pr} \gls{auc} did not reproduce these findings
(all returned $\text{IoU} < 0.1$; see \Cref{apx:sec:subgroup_injection}).
The injected subgroup contains mostly negatives, which interacts with the class balance skews of \gls{arl} and \gls{pr} \gls{auc} (see \Cref{subsubsec:class_ratio_skew_analysis}).
Furthermore, the injected subgroup is small, making it inherently hard to capture with \gls{arl}, which has an embedded size skew (see \Cref{subsubsec:cover_size_skew_analysis}).
This highlights the importance of the size and class balance terms introduced in SubROC to counter these skews, but also shows that they do not completely remove the effects.
Future work, for example, employing a normalized \gls{pr} \gls{auc} \cite{boyd_unachievable_2012} may provide further improvement.

\subsection{Importance of Cover Size and Class Imbalance Weighting}
\label{subsec:skew_analysis}

To illustrate the effect of the interaction between model performance, size and class ratio on the proposed scoring functions, we conduct experiments on synthetic data (cf. \cite{cook_when_2020}).
Data was generated from a bivariate normal distribution with means $0$ and variances $1$.
The first variable $z$ was used to derive a ground truth label, the second $\hat{z}$ represents the prediction values.
Varying their correlation thus varies the ``prediction performance'' (the $y$-axes). 
Dataset sizes (cover of a subgroup) are varied in \Cref{fig:performance_measure_skews_cover_size} ($x$ axis) and set to $100$ in \Cref{fig:performance_measure_skews_ncr}. 
Class balance (NCR) $q$ was varied by thresholding $z$ at the first $q$-quantile in \Cref{fig:performance_measure_skews_ncr} ($x$ axis) and fixed at $q=0.5$ for \Cref{fig:performance_measure_skews_cover_size}. 
Since scoring functions are relative to the overall dataset, we equivalently derive a dataset with correlation $0$, size $100$ and $q=0.5$.
The generation process was repeated $20$ times followed by an averaging of scores to reduce noise.
Note that this procedure allows to evaluate practically impossible subgroups, since any subgroup with cover size or \gls{ncr} close to that of the overall dataset can only have limited exceptionality (i.e. performance difference) in practice.

\subsubsection{Skew from Cover Size}
\label{subsubsec:cover_size_skew_analysis}

\Cref{fig:arl_size_dependency} shows the dependence of $\varphi^{rasl}$ on the number of instances in its input, yielding larger values for larger inputs even if predictive performance is kept constant (see the horizontal gradient at fixed correlations).
This can lead to low values of $\varphi^{rasl}$ on small subgroups, which do not necessarily correspond to a performance improvement over the complete dataset.
Underperforming subgroups therefore cannot be identified by a fixed threshold on the metric (e.g., $\geq 0$).
Therefore smaller subgroups have to have a higher error rate to be considered important for \gls{arl}, and thus the approach returns results with a higher rate of passing the statistical test (\Cref{tab:generalizability_metrics_overview_part_1}).
However, the strength of this effect can not be controlled by the user.

Neither $\varphi^{rPRAUC}$ nor $\varphi^{rROCAUC}$ change with cover size (\ref{fig:roc_auc_size_dependency},  \ref{fig:pr_auc_size_dependency}).
Furthermore these figures show increased noise on the lower end of the cover size axes, indicating increased random effects on these scoring functions for subgroups with few instances which underlines the necessity of the statistical tests introduced by SubROC.
Additionally, \Cref{fig:roc_auc_weighted_size_dependency,fig:pr_auc_weighted_size_dependency} show that the introduced cover size weighting leads to a dependence on input size, encoding a preference for larger inputs for both of these scoring functions.

\input{skew_plots_combined_figure_1}

\input{skew_plots_combined_figure_2}

\subsubsection{Skew from Class Ratio}
\label{subsubsec:class_ratio_skew_analysis}

\Cref{fig:arl_class_balance_dependency,fig:pr_auc_class_balance_dependency} show that the \gls{arl} and \gls{pr} \gls{auc} depend on the class ratio.
The \gls{pr} \gls{auc} skew is known to be related to the unachievable area in \gls{pr} space \cite{boyd_unachievable_2012}.
This skew has a similar influence on the interpretation of the scores as has the cover size skew of the \gls{arl}.
\Cref{fig:roc_auc_class_balance_dependency} shows a more subtle effect for the \gls{roc} \gls{auc} through its contour lines, that was also found by \cite{cook_when_2020}.
\Cref{fig:arl_weighted_class_balance_dependency,fig:roc_auc_weighted_class_balance_dependency,fig:pr_auc_weighted_class_balance_dependency} show that the class balance weighting introduces a preference towards subgroups with balanced classes.
In particular, the score of subgroups with highly imbalanced classes is close to zero.
Since \gls{roc} \gls{auc} exhibits the least prominent skews (cf. \Cref{fig:roc_auc_size_dependency,fig:roc_auc_class_balance_dependency}), searches utilizing the \gls{roc} \gls{auc} within SubROC can be expected to yield the most consistently underperforming subgroups over the widest range of cover sizes and class balances compared to the \gls{arl} and \gls{pr} \gls{auc}.

%% file: table_result_quality_summary_depth_4_manually_rearranged_main_text_merged.tex
\footnotesize
\begin{tabular}{ll?r|r?r|r@{ / }r@{ / }r?r|r?r|r?r|r}
 &  
    & \multicolumn{2}{c?}{\rotatebox{90}{\makecell[l]{Mean\\Cover\\Size}}} 
    & \multicolumn{1}{c}{\rotatebox{90}{\makecell[l]{Significant /\\Top-5}}} 
    & \multicolumn{3}{c?}{\rotatebox{90}{\makecell[l]{Filtered-5 /\\Significant / \\Top-100}}}
    & \multicolumn{2}{c?}{\rotatebox{90}{\makecell[l]{Mean\\Exception-\\ality}}} 
    & \multicolumn{2}{c?}{\rotatebox{90}{\makecell[l]{Mean\\NCR}}} 
    & \multicolumn{2}{c}{\rotatebox{90}{\makecell[l]{Mean\\Pairwise\\IoU}}} \\
Dataset & 
    \makecell{Metric} & 
    \makecell[c]{b} & 
    \makecell[c]{f} &
    \makecell[c]{b} & 
    \multicolumn{3}{c?}{f} &
    \makecell[c]{b} & 
    \makecell[c]{f} &
    \makecell[c]{b} & 
    \makecell[c]{f} &
    \makecell[c]{b} & 
    \makecell[c]{f} \\
\midrule
\multirow[c]{3}{*}{\shortstack[l]{\textbf{Census}\\n=199523\\NCR=0.06}} & 
 ROC
    & 25 
    & \textbf{6753}
    & 1 / 5 
    & \textbf{5} & \textbf{99} & \textbf{100}
    & \textbf{0.95} & 0.05 
    & 0.04 & \textbf{0.16}  
    & \textbf{0.10} & 0.31
\\
 \cline{2-14} & 
 PR & 
    23 & \textbf{864} & 
    0 / 5 & \textbf{5} & \textbf{99} & \textbf{100} & 
    \textbf{0.92} & 0.14 & 
    0.93 & \textbf{0.45} & 
    \textbf{0.37} & 0.45\\
 \cline{2-14} & 
 ARL & 
    8000 & \textbf{11945} & 
    5 / 5 & 5 & 27 & 27 & 
    \textbf{47} & 31 & 
    \textbf{0.15} & 0.11 & 
    0.99 & \textbf{0.50} \\
\cline{1-14}
\multirow[c]{3}{*}{\shortstack[l]{\textbf{Adult}\\n=48842\\NCR=0.24}} & 
ROC & 
    27 & \textbf{6492} & 
    0 / 5 & \textbf{5} & \textbf{77} & \textbf{100} & 
    \textbf{0.92} & 0.06 & 
    0.04 & \textbf{0.37} & 
    \textbf{0.48} & 0.61 \\
 \cline{2-14}
 & PR & 
    27 & \textbf{2803} & 
    2 / 5 & \textbf{5} & \textbf{70} & \textbf{100} & 
    \textbf{0.94} & 0.08 & 
    0.96 & \textbf{0.46} & 
    0.37 & \textbf{0.19} \\
 \cline{2-14}
 & \multirow[c]{1}{*}{ARL} & 
    4614 & \textbf{6930} & 
    5 / 5 & 5 & 5 & 5 & 
    \textbf{148} & 55 & 
    \textbf{0.42} & 0.34 & 
    0.85 & \textbf{0.61} \\
\cline{1-14}
\multirow[c]{3}{*}{\shortstack[l]{\textbf{Bank}\\n=45211\\NCR=0.88}} & 
ROC & 
    22 & \textbf{311} & 
    0 / 5 & \textbf{5} & \textbf{22} & \textbf{100} & 
    \textbf{0.83} & 0.16 & 
    0.95 & \textbf{0.48} & 
    \textbf{0.10} & 0.12 \\
 \cline{2-14}
 & PR & 
    \textbf{157} & 72 & 
    0 / 5 & \textbf{2} & \textbf{2} & \textbf{100} & 
    \textbf{0.68} & 0.10 & 
    0.99 & \textbf{0.48} & 
    0.80 & \textbf{0.51} \\
 \cline{2-14}
 & ARL & 
    - & - & 
    0 / 0 & 0 & 0 & 0 & 
    - & - & 
    - & - & 
    - & - \\
\cline{1-14}
\multirow[c]{3}{*}{\shortstack[l]{\textbf{Credit}\\n=30000\\NCR=0.78}} 
& ROC & 
    24 & \textbf{1506} & 
    0 / 5 & \textbf{5} & \textbf{38} & \textbf{100} & 
    \textbf{0.75} & 0.12 & 
    0.96 & \textbf{0.61} & 
    \textbf{0.01} & 0.09 \\
 \cline{2-14}
 & PR & 
    75 & \textbf{1895} & 
    0 / 5 & \textbf{5} & \textbf{17} & \textbf{100} & 
    \textbf{0.51} & 0.29 & 
    0.99 & \textbf{0.85} & 
    \textbf{0.18} & 0.34 \\
 \cline{2-14}
 & ARL & 
    - & - & 
    0 / 0 & 0 & 0 & 0 & 
    - & - & 
    - & - & 
    - & - \\
\cline{1-14}
\multirow[c]{3}{*}{\shortstack[l]{\textbf{Approval}\\n=690\\NCR=0.55}} & 
    ROC & 
    29 & - & 
    0 / 5 & 0 & 0 & 41 & 
    0.64 & - & 
    0.96 & - & 
    0.69 & - \\
 \cline{2-14}
 & PR & 
    41 & \textbf{80} & 
    \textbf{2 / 5} & 1 & 1 & 25 & 
    \textbf{0.89} & 0.85 & 
    0.98 & \textbf{0.95} & 
    0.75 & - \\
 \cline{2-14}
 & ARL & 
    60 & \textbf{87} & 
    3 / 5 & 3 & 3 & 3 & 
    \textbf{26} & 10 & 
    0.96 & \textbf{0.79} & 
    \textbf{0.68} & 0.69 \\
\end{tabular}

%% file: table_case_study_filtered_result_set_table_1_1_sklearn.metrics.roc_auc_score_True_reduced.tex
\begin{tabular}{@{}l@{\hspace{2mm}}r@{\hspace{2mm}}r@{\hspace{2mm}}r@{}}
\toprule
Pattern & ROC AUC & Cover & NCR \\
\midrule
\makecell[l]{$\text{marital-status}=\text{Married-civ-spouse}$} & \makecell[r]{0.8508} & \makecell[r]{5280} & \makecell[r]{0.4496} \\
\makecell[l]{$\text{relationship}=\text{Husband}$} & \makecell[r]{0.8485} & \makecell[r]{4638} & \makecell[r]{0.4495} \\
\makecell[l]{$\text{capital-gain}\in[0.0,114.0)\:\wedge$\\$\text{marital-status}=\text{Married-civ-spouse}$} & \makecell[r]{0.8053\\\ } & \makecell[r]{4628\\\ } & \makecell[r]{0.4114\\\ } \\
\makecell[l]{$\text{sex}=\text{Male}$} & \makecell[r]{0.9053} & \makecell[r]{7581} & \makecell[r]{0.3109} \\
\makecell[l]{$\text{capital-gain}\in[0.0,114.0)$} & \makecell[r]{0.9015} & \makecell[r]{10331} & \makecell[r]{0.2126} \\
\midrule
\makecell[l]{$\emptyset$} & \makecell[r]{0.9241} & \makecell[r]{11305} & \makecell[r]{0.2482} \\
\bottomrule
\end{tabular}

%% file: table_optimistic_estimates_xgboost_depth_4_time_median_speedup_table_main_split.tex
\begin{tabular}{lr?r|r|r|r|r}
\rotatebox{90}{\makecell[l]{Performance\\Measure}}
 & \rotatebox{90}{\makecell[l]{Weight ($\alpha=\beta$)}}
 & \rotatebox{90}{\makecell[l]{\textbf{Census}\\n/m=199523/41}} 
 & \rotatebox{90}{\makecell[l]{\textbf{Adult}\\n/m=48842/14}} 
 & \rotatebox{90}{\makecell[l]{\textbf{Bank}\\n/m=45211/16}} 
 & \rotatebox{90}{\makecell[l]{\textbf{Credit}\\n/m=30000/23}} 
 & \rotatebox{90}{\makecell[l]{\textbf{Approval}\\n/m=690/15}} \\

\midrule
\multirow[c]{4}{*}{ARL} & 0 & 9.1 & 11.6 & 482.0 & 10.2 & 49.3 \\
 & 0.1 & 9.3 & 11.1 & 333.9 & 9.3 & 24.7 \\
 & 0.3 & 10.4 & 13.0 & 337.6 & 9.5 & 24.6 \\
 & 1 & 15.0 & 27.5 & 333.3 & 9.3 & 25.4 \\
\cline{1-7}
\multirow[c]{4}{*}{PR AUC} & 0 & 2.0 & 1.5 & 460.7 & 3.3 & 43.8 \\
 & 0.1 & 2.0 & 1.2 & 337.2 & 1.7 & 25.4 \\
 & 0.3 & 2.0 & 1.2 & 338.0 & 1.6 & 24.6 \\
 & 1 & 3.8 & 5.3 & 337.1 & 2.2 & 24.9 \\
\cline{1-7}
\multirow[c]{4}{*}{ROC AUC} & 0 & 1.5 & 1.1 & 401.7 & 1.7 & 41.7 \\
 & 0.1 & 1.5 & 1.0 & 332.8 & 1.6 & 17.0 \\
 & 0.3 & 1.4 & 1.0 & 303.2 & 1.6 & 19.5 \\
 & 1 & 2.8 & 4.4 & 307.4 & 1.8 & 24.2 \\
\end{tabular}

%% file: figure_injection_experiment_combined_roc_auc.tex
\begin{figure*}[t]
    \centering
    \begin{subfigure}[]{0.24\linewidth}
        \centering
        \includegraphics[]{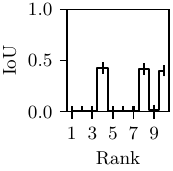}
        \caption{baseline}
    \end{subfigure}
    \begin{subfigure}[]{0.24\linewidth}
        \centering
        \includegraphics[]{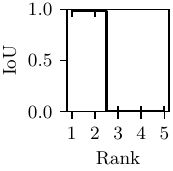}
        \caption{$\alpha=0,\beta=0$}
    \end{subfigure}
    \begin{subfigure}[]{0.24\linewidth}
        \centering
        \includegraphics[]{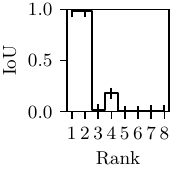}
        \caption{$\alpha=0.3,\beta=0$}
    \end{subfigure}
    \begin{subfigure}[]{0.24\linewidth}
        \centering
        \includegraphics[]{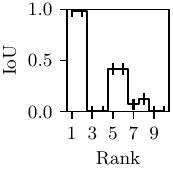}
        \caption{$\alpha=0,\beta=0.3$}
    \end{subfigure}
    \caption{
    Intersection over Union (IoU) of subgroup covers, comparing an injected subgroup with subgroups in a top-10 \gls{roc} \gls{auc} result set, ordered decreasingly by interestingness.
    \enquote{baseline} was obtained searching in a base SubROC setting. The other results were obtained using the full SubROC framework with corresponding cover size weighting $\alpha$ and class balance weighting $\beta$, successfully discovering the injected subgroup.
    The full SubROC framework accurately found the injected subgroup for any weighting configuration by returning two highly overlapping subgroups in the highest ranks of all result sets.
    }
    \label{fig:subgroup_injection_result_roc_auc}
\end{figure*}

%% file: skew_plots_combined_figure_1.tex
\begin{figure*}[ht!]
    \centering
    \begin{subfigure}[b]{0.14\linewidth}
        \centering
        \includegraphics[height=2.45cm]{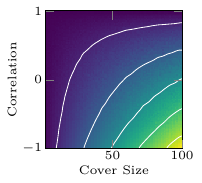}
        \caption{
            $\varphi^{rasl}_{0, 0}$
        }
        \label{fig:arl_size_dependency}
    \end{subfigure}
    \begin{subfigure}[b]{0.14\linewidth}
        \centering
        \includegraphics[height=2.4cm]{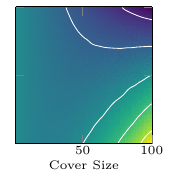}
        \caption{
            $\varphi^{rasl}_{1, 0}$
        }
        \label{fig:arl_weighted_size_dependency}
    \end{subfigure}
    \begin{subfigure}[b]{0.14\linewidth}
        \centering
        \includegraphics[height=2.4cm]{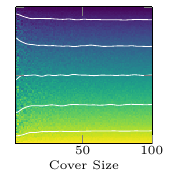}
        \caption{
            $\varphi^{rROCAUC}_{0, 0}$
        }
        \label{fig:roc_auc_size_dependency}
    \end{subfigure}
    \begin{subfigure}[b]{0.14\linewidth}
        \centering
        \includegraphics[height=2.4cm]{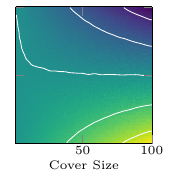}
        \caption{
            $\varphi^{rROCAUC}_{1, 0}$
        }
        \label{fig:roc_auc_weighted_size_dependency}
    \end{subfigure}
    \begin{subfigure}[b]{0.14\linewidth}
        \centering
        \includegraphics[height=2.4cm]{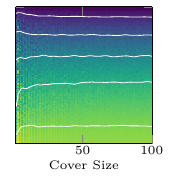}
        \caption{
            $\varphi^{rPRAUC}_{0, 0}$
        }
        \label{fig:pr_auc_size_dependency}
    \end{subfigure}
    \begin{subfigure}[b]{0.14\linewidth}
        \centering
        \includegraphics[height=2.4cm]{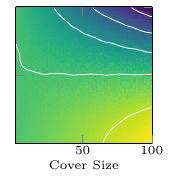}
        \caption{
            $\varphi^{rPRAUC}_{1, 0}$
        }
        \label{fig:pr_auc_weighted_size_dependency}
    \end{subfigure}
    \begin{subfigure}[b]{0.08\linewidth}
        \centering
        \includegraphics[trim={47mm 0 0 0},clip,height=2.4cm]{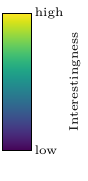}
        \vspace{2.5mm}
    \end{subfigure}
    \caption{
        Skewed response of cover size and class balance weighted scoring functions depending on cover size and correlation (model performance) on synthetic data.
        Contour lines show smoothed paths along which interestingness is equal.
        Unweighted scoring functions based on the \gls{auc} performance measures show increased noise on subgroups with small covers.
        Significant dependencies on the cover size are visible in the results for the unweighted scoring function based on the \gls{arl} performance measure as well as all weighted scoring functions.
    }
    \label{fig:performance_measure_skews_cover_size}
\end{figure*}

%% file: skew_plots_combined_figure_2.tex
\begin{figure*}[ht!]
    \centering
    \begin{subfigure}[b]{0.14\linewidth}
        \centering
        \includegraphics[height=2.45cm]{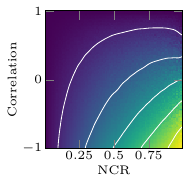}
        \caption{
            $\varphi^{rasl}_{0, 0}$
        }
        \label{fig:arl_class_balance_dependency}
    \end{subfigure}
    \begin{subfigure}[b]{0.14\linewidth}
        \centering
        \includegraphics[height=2.4cm]{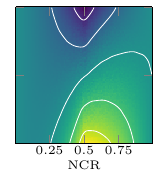}
        \caption{
            $\varphi^{rasl}_{0, 1}$
        }
        \label{fig:arl_weighted_class_balance_dependency}
    \end{subfigure}
    \begin{subfigure}[b]{0.14\linewidth}
        \centering
        \includegraphics[height=2.4cm]{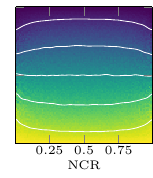}
        \caption{
            $\varphi^{rROCAUC}_{0, 0}$
        }
        \label{fig:roc_auc_class_balance_dependency}
    \end{subfigure}
    \begin{subfigure}[b]{0.14\linewidth}
        \centering
        \includegraphics[height=2.4cm]{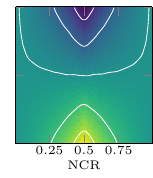}
        \caption{
            $\varphi^{rROCAUC}_{0, 1}$
        }
        \label{fig:roc_auc_weighted_class_balance_dependency}
    \end{subfigure}
    \begin{subfigure}[b]{0.14\linewidth}
        \centering
        \includegraphics[height=2.4cm]{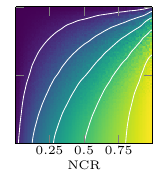}
        \caption{
            $\varphi^{rPRAUC}_{0, 0}$
        }
        \label{fig:pr_auc_class_balance_dependency}
    \end{subfigure}
    \begin{subfigure}[b]{0.14\linewidth}
        \centering
        \includegraphics[height=2.4cm]{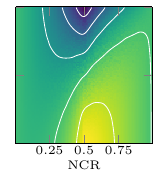}
        \caption{
            $\varphi^{rPRAUC}_{0, 1}$
        }
        \label{fig:pr_auc_weighted_class_balance_dependency}
    \end{subfigure}
    \begin{subfigure}[b]{0.08\linewidth}
        \centering
        \includegraphics[trim={47mm 0 0 0},clip,height=2.4cm]{skew_plots_colorbar.pdf}
        \vspace{2.5mm}
    \end{subfigure}
    \caption{
        Skewed response of cover size and class balance weighted scoring functions depending on negative class ratio and correlation (model performance) on synthetic data.
        Contour lines show smoothed paths along which interestingness is equal.
        The unweighted scoring functions based on the \gls{arl} and the \gls{pr} \gls{auc} both increase with the \gls{ncr} for almost any correlation.
        The unweighted scoring function based on the \gls{roc} \gls{auc} depends on the difference between the \gls{ncr} and $0.5$.
        All weighted scoring functions return the most extreme values close to \gls{ncr}s of $0.5$, with a gradient from low to high scores as correlations decrease.
    }
    \label{fig:performance_measure_skews_ncr}
\end{figure*}

%% file: related_work.tex
We summarize work on producing descriptions for parts of a dataset where a soft classifier over- or underperforms and describe the differences to our approach.

Sagadeeva and Boehm \cite{sagadeeva_sliceline_2021} only consider residuals-based loss functions to rank subgroups in contrast to our ranking-based measures.
Similarly, Zhang and Neill \cite{zhang_identifying_2017} use neither \gls{auc} metrics nor the subgroup discovery framework.
Other research allows for user-specified performance measures for subgroup ranking, but they do not provide specific support such as optimistic estimates or generalization-awareness\cite{chung_automated_2020,zhang_sliceteller_2017}.
\cite{cabrera_fairvis_2019,chung_automated_2020,zhang_sliceteller_2017} propose automatic search approaches as part of generic interactive evaluation tools.
None of these works applied the difference between the subgroup \gls{auc} to the overall dataset \gls{auc} for the \gls{roc} or \gls{pr} curve in a subgroup discovery setting and investigated its properties in detail.

The following works are concerned with the \gls{roc} curve of subgroups in detail.
\cite{gardner_evaluating_2019} propose to use the area between the \gls{roc} curves of a model on two subgroups in the context of fairness of data slicing but do not discuss the mining of respective patterns.
\cite{carrington_deep_2023} propose to evaluate specific parts of the \gls{roc} curve and call these parts subgroups.
Considered parts in that work are intervals in the x-axis or the y-axis of the \gls{roc} space as well as intervals in the thresholds used to compute the \gls{roc} points from.
In addition to this differing notion of subgroups, they do not investigate an automated search for subgroups.
\cite{tafvizi_attributing_2022} presents methods for computing the attribution of subgroups to the overall \gls{auc}.
They do this to identify parts of the data that contribute to a reduction in the \emph{overall} \gls{auc}.
The contribution of a subgroup to the overall dataset \gls{auc} is highly influenced by where its instances fall within the global ranking of predictions for the entire dataset.
Therefore finding parts of the dataset with an exceptional \gls{auc} would simply not contribute to their goal.
Our goal is to find under which circumstances a model is most likely to fail, so we take a different approach with \gls{auc}-based exceptionality.

For subgroup discovery, many (tight) optimistic estimates (OE) have been proposed, primarily for settings with traditional target concepts (binary, numeric, or both) \cite{grosskreutz_fast_2010,lemmerich_fast_2016}.
In a few settings, tight OE have also been derived for more complex model classes in the EMM framework
\cite{kalofolias_discovering_2019,belfodil_identifying_2020}.
An OE for a scoring function that computes the \gls{roc} \gls{auc} is given in \cite{lemmerich_fast_2016}.
However, there the \gls{roc} curve is based on different attributes compared to ours, such that the OE is not applicable here.
To the best of our knowledge, no previously presented OE can be applied to the scoring functions used in this work.

%% file: conclusion.tex
We introduced  SubROC, a novel framework for reliable and efficient discovery of interpretable subgroups with exceptional model performance for soft classifiers.
Our experiments demonstrate the necessity for such a framework as well as the effectiveness of SubROC when using the performance measures \gls{arl}, \gls{roc} \gls{auc} and \gls{pr} \gls{auc}.
We identified the side effects of a subgroup's cover size and class balance on these performance measures and their intricate influence on the results.
By providing optimistic estimates, we enable fast mining.
To summarize, SubROC has the potential to enable routine evaluation of ML models, improve quality control procedures, inform model performance enhancements, and increase confidence in ML models.

In this work, we focused on ranking-based soft classifier evaluation measures given their widespread use and the specific challenges they pose for integration into the Exceptional Model Mining framework.
By contrast, other popular performance measures, such as cross-entropy or the Brier score decompose over individual instances and can therefore be more directly integrated into classic subgroup discovery.
Their in-depth evaluation is an avenue for future research.
Also, directly integrating statistical tests into the search may help to further improve the significance of the discovered subgroups.
Finally, the investigation of alternatives to account for class imbalances is a promising direction for future work.

%% file: supplementary_material.tex
\crefalias{section}{appendix}

\section{Proofs of Tight Optimistic Estimates}
\label{apx:proofs_of_tight_optimistic_estimates}

Firstly this section generalizes and formalizes the interestingness measure $\varphi^{rasl}$.
Generic interestingness measures are defined in the following, that allow specialization by soft classifier performance measures.
Here soft classifier performance measures are functions that assign a real value to ground truth and prediction pairs of soft classifiers.

Our formalizations make use of standard multiset notation as presented in \cite{syropoulos_mathematics_2001}.
To handle values of instances in a dataset $D=(\mathcal{I}, \mathcal{A})$ we define a \emph{database} as the multiset $d(D) = [(A_1(c), \dots, A_{|\mathcal{A}|}(c)) | c \in \mathcal{I}]$.
A shorthand notation for ground truth and prediction values is $I^{\{y, \hat{y}\}} = [(Y(c), \hat{Y}(c)) | c \in I]$ on a set of instances $I$, the item values of which are referenced by $y_i$ and $\hat{y}_i$ with index $i$.

For the start, the aim of the presented interestingness measures is to assign the highest interestingness to the subgroups on which the model exhibits the worst performance.
The type of a performance measure characterizes whether a higher value indicates better (score-type) or worse (loss-type) performance.
A generic interestingness measure is defined for each performance measure type, negating score-type performance measure values such that higher values indicate worse performance.

\begin{definition}[Soft Classifier Loss Interestingness Measure] \label{def:loss_interestingness_measure}
    Let $D = (\mathcal{I}, \mathcal{A})$ be a dataset.
    Let $m_l$ be a loss-type soft classifier performance measure.
    The soft classifier loss interestingness measure for $m_l$ is defined as
    \begin{equation*}
        \varphi^{m_l}(p) = m_l(sg(p)^{\{y, \hat{y}\}})
    \end{equation*}
    for any pattern $p$ that is defined on $D$.
\end{definition}

\begin{definition}[Soft Classifier Score Interestingness Measure] \label{def:score_interestingness_measure}
    Let $D = (\mathcal{I}, \mathcal{A})$ be a dataset.
    Let $m_s$ be a score-type soft classifier performance measure.
    The soft classifier score interestingness measure for $m_s$ is defined as
    \begin{equation*}
        \varphi^{m_s}(p) = -m_s(sg(p)^{\{y, \hat{y}\}})
    \end{equation*}
    for any pattern $p$ that is defined on $D$.
\end{definition}

These interestingness measures both increase with decreasing exhibited performance on the input data.
It remains unclear at which interestingness value a subgroup exceeds the performance (positively or negatively) compared to the overall analyzed dataset.
An example application of that information is to prune the search for subgroups exhibiting exceptionally good model performance, when the aim is to find subgroups exhibiting exceptionally bad model performance.
The following interestingness measure definitions build on the previous definitions to include the comparison to the overall dataset performance.

\begin{definition}[Relative Soft Classifier Loss Interestingness Measure] \label{def:relative_loss_interestingness_measure}
    Let $D = (\mathcal{I}, \mathcal{A})$ be a dataset.
    Let $\varphi^{m_l}$ be a soft classifier loss interestingness measure for the loss-type soft classifier performance measure $m_l$.
    The relative soft classifier loss interestingness measure is defined as
    \begin{equation*}
        \varphi^{rm_l}(p) = \varphi^{m_l}(p) - m_l(\mathcal{I}^{\{y, \hat{y}\}})
    \end{equation*}
    for any pattern $p$ that is defined on $D$.
\end{definition}

\begin{definition}[Relative Soft Classifier Score Interestingness Measure] \label{def:relative_score_interestingness_measure}
    Let $D = (\mathcal{I}, \mathcal{A})$ be a dataset.
    Let $\varphi^{m_s}$ be a soft classifier score interestingness measure for the score-type soft classifier performance measure $m_s$.
    The relative soft classifier score interestingness measure is defined as
    \begin{equation*}
        \varphi^{rm_s}(p) = \varphi^{m_s}(p) + m_s(\mathcal{I}^{\{y, \hat{y}\}})
    \end{equation*}
    for any pattern $p$ that is defined on $D$.
\end{definition}

The main interestingness measure $\varphi^{rasl}$ from \cite{duivesteijn_understanding_2014} is defined with alternative signatures (compared to the main text) for use in the proofs as
\begin{align*}
     PEN_i([c_1, \dots, c_m]) := &\sum_{j=1}^{m} (\1_{\{(y, \hat{y}) \mid y = 0 \wedge \hat{y} > \hat{y}_i\}} (y_j, \hat{y}_j) \\
     &+ \frac{1}{2} \cdot \1_{\{(y, \hat{y}) \mid y = 0 \wedge \hat{y} = \hat{y}_i\}} (y_j, \hat{y}_j)) \\
     RL_i([c_1, \dots, c_m]) := &\1_{\{1\}} (y_i) \cdot PEN_i([c_1, \dots, c_m]) \\
     ARL([c_1, \dots, c_m]) := &\frac{\sum_{i=1}^{m} \1_{\{1\}} (y_i) \cdot PEN_i([c_1, \dots, c_m])}{\sum_{i=1}^{m} \1_{\{1\}} (y_i)} \\ 
    \varphi^{rasl}(p) := &ARL(sg(p)^{\{y, \hat{y}\}}) - ARL(\mathcal{I}^{\{y, \hat{y}\}}) \\
\end{align*}
for multisets of instance values $c_i \in \{0, 1\} \times \R$.
$\varphi^{rasl}$ is a special case of \Cref{def:relative_loss_interestingness_measure} with $m_l = ARL$.
The definitions from the main text can be formalized on this basis as
\begin{align*}
    PEN(c, p) := &PEN_{index(c)}(sg(p)^{\{y, \hat{y}\}}) \\
    ARL(p) := &ARL(sg(p)^{\{y, \hat{y}\}})
\end{align*}
for patterns $p$ and instances $c$.

$\varphi^{ROCAUC}$ and $\varphi^{PRAUC}$ are both instances of \Cref{def:score_interestingness_measure} with the \gls{roc} \gls{auc} and the \gls{pr} \gls{auc} performance measures substituted in place of $m_s$ respectively.
$\varphi^{rROCAUC}$ and $\varphi^{rPRAUC}$ are the same for \Cref{def:relative_score_interestingness_measure}.

Note that when referring to the \gls{auc} of the \gls{pr} curve in this work, we mean the area under the linearly interpolated \gls{pr} curve.
Therefore \gls{roc} \gls{auc} and \gls{pr} \gls{auc} are defined as follows.

We define the count of positives and negatives in $C = [c_1, \dots, c_n]$ (multiset of instance values $c_i \in \{0, 1\} \times \mathbb{R}$) as
$$P(C) := |[(y_i, \hat{y}_i) \in C \mid y_i = 1]|$$
and 
$$N(C) := |[(y_i, \hat{y}_i) \in C \mid y_i = 0]|$$
respectively.

Furthermore we define multisets of true positives, false positives, true negatives and false negatives in $C = [c_1, \dots, c_n]$ (multiset of instance values $c_i \in \{0, 1\} \times \mathbb{R}$) under threshold $t \in \mathbb{R}$ as
\begin{align*}
TP(C, t) &:= [(y_i, \hat{y}_i) \in C \mid y_i = 1 \wedge \hat{y}_i > t]\\
FP(C, t) &:= [(y_i, \hat{y}_i) \in C \mid y_i = 0 \wedge \hat{y}_i > t]\\
TN(C, t) &:= [(y_i, \hat{y}_i) \in C \mid y_i = 0 \wedge \hat{y}_i \leq t]\\
FN(C, t) &:= [(y_i, \hat{y}_i) \in C \mid y_i = 1 \wedge \hat{y}_i \leq t]
\end{align*}
respectively.

We also define the set of \gls{roc}/\gls{pr} supporting points as
$$S(C) := \{(|TP(C, t)|, |FP(C, t)|) \mid t \in \{\hat{y}_i \mid (y_i, \hat{y}_i) \in C\} \cup \{-\infty\}\},$$
which consists of (true positives, false positives) pairs constructed from applying unique $\hat{y}_i$ from $C$ as classification thresholds on $C$.

With that we define \gls{roc} \gls{auc} and \gls{pr} \gls{auc} as follows.

\begin{definition}[ROCAUC \cite{fawcett_introduction_2006}]
    The \gls{roc} \gls{auc} is defined as
    $$ROCAUC(C) := \frac{\sum_{i=2}^{|S(C)|} \mid FP_i - FP_{i-1} \mid \cdot \frac{TP_i + TP_{i-1}}{2}}{P(C) \cdot N(C)}$$
    for any $C = [c_1, \dots, c_n]$ (multiset of instance values $c_i \in \{0, 1\} \times \mathbb{R}$) with $S(C)$ indexed by the decreasing order of thresholds.
\end{definition}

\begin{definition}[PRAUC]
    The (linear!) \gls{pr} \gls{auc} is defined as
    $$PRAUC(C) := \frac{\sum_{i=2}^{|S(C)|} \mid TP_i - TP_{i-1} \mid \cdot \frac{\frac{TP_i}{TP_i + FP_i} + \frac{TP_{i-1}}{TP_{i-1}+FP_{i-1}}}{2}}{P(C)}$$
    for any $C = [c_1, \dots, c_n]$ (multiset of instance values $c_i \in \{0, 1\} \times \mathbb{R}$) with $S(C)$ indexed by the decreasing order of thresholds.
\end{definition}

Both of these definitions are derived from the trapezoidal rule, which corresponds to the area under the linear interpolation of the set of points.
The underlying performance measures Precision and Recall as well as a general \gls{auc} definition based on the trapezoidal rule read as follows.

We define the Precision and Recall performance measures on $C = [c_1, \dots, c_n]$ (multiset of instance values $c_i \in \{0, 1\} \times \mathbb{R}$) under threshold $t \in \mathbb{R}$ as
\begin{align*}
    Prec(C, t) &:= \frac{|TP(C, t)|}{|TP(C, t)| \cdot |FP(C, t)|}\\
    Rec(C, t) &:= \frac{|TP(C, t)|}{P(C)}
\end{align*}
respectively.

\begin{definition}[AUC]
    The \gls{auc} of a set $\mathfrak{P}$ of at least two two-dimensional points $x_i = (x_{i, 1}, x_{i, 2})$ (with indices $i$ such that points are in increasing order, primarily for the first dimension and secondarily for the second dimension) is defined using the trapezoidal rule as
    $$AUC(\mathfrak{P}) := \sum_{i=2}^{|\mathfrak{P}|} \left | x_{i, 1} - x_{i-1, 1} \right | \cdot \frac{x_{i, 2} + x_{i-1, 2}}{2}.$$
\end{definition}

The \gls{roc} \gls{auc} and \gls{pr} \gls{auc} functions as they are used in the main text have different signatures.
They can be formalized based on the definitions here as
\begin{align*}
    ROCAUC(I) := &ROCAUC(I^{\{y, \hat{y}\}})\\
    PRAUC(I) := &PRAUC(I^{\{y, \hat{y}\}})
\end{align*}
for sets of instances $I$.

On a side note, the class balance can now be defined on multisets of ground truth and prediction values $C = I^{\{y, \hat{y}\}}$ for sets of instances $I$.
By definition $P(C) = |\mathcal{P}_I|$ and $N(C) = |\mathcal{N}_I|$.
We define $cb(C)$ as 0 if only one class occurs and as $\min\{\frac{N(C)}{P(C)}, \frac{P(C)}{N(C)}\}$ otherwise.

Definitions of (tight) lower bounds of soft classifier performance measures are now given.
They allow to structure the construction of some of the tight optimistic estimates that are presented in this section into a statement about a performance measure, considering a bound for the performance measure, and an adaptation step of that bound to interestingness measures, yielding an optimistic estimate.
The lower bound definition (\Cref{def:performance_measure_lower_bound}) requires for a soft classifier performance measure an additional soft classifier performance measure that does not return smaller values than the reference performance measure for any input.
The tightness criterion of \Cref{def:performance_measure_tight_bound} requires a lower bound to be equal to the reference performance measure for at least one input.
The performance measure lower bound definition and the bound tightness definition are defined analogously to optimistic estimates and tight optimistic estimates respectively.

\begin{definition}[Soft Classifier Performance Measure Lower Bound] \label{def:performance_measure_lower_bound}
    Let $m$ and $b_m$ be soft classifier performance measures.
    $b_m$ is the soft classifier performance measure lower bound of $m$ if and only if 
    \begin{equation*}
        \forall C' \subseteq C: m(C') \geq b_m(C)
    \end{equation*}
    for any $C = [c_1, \dots, c_n]$ (multiset of instance values $c_i \in \{0, 1\} \times \mathbb{R}$).
\end{definition}

\begin{definition}[Tight Soft Classifier Performance Measure Bound] \label{def:performance_measure_tight_bound}
    Let $m$ be a soft classifier performance measure with soft classifier performance measure lower bound $b_m$.
    $b_m$ is called tight if and only if
    \begin{equation*}
        \exists C' \subseteq C: m(C') = b_m(C)
    \end{equation*}
    for any $C = [c_1, \dots, c_n]$ (multiset of instance values $c_i \in \{0, 1\} \times \mathbb{R}$).
\end{definition}

\begin{definition}[Soft Classifier Performance Measure Upper Bound] \label{def:metric_upper_bound}
    Let $m$ and $b_m$ be soft classifier performance measures.
    $b_m$ is the soft classifier performance measure upper bound of $m$ if and only if
    \begin{equation*}
        \forall C' \subseteq C: m(C') \leq b_m(C)
    \end{equation*}
    for any $C = [c_1, \dots, c_n]$ (multiset of instance values $c_i \in \{0, 1\} \times \mathbb{R}$).
\end{definition}

\input{theorem_transition_loss_interestingness_measure}
\input{proof_transition_loss_interestingness_measure}

\input{theorem_transition_score_interestingness_measure}
\input{proof_transition_score_interestingness_measure}

The following lemma allows to translate an optimistic estimate for a soft classifier loss (or score) interestingness measure to an optimistic estimate for a \textit{relative} soft classifier loss (or score) interestingness measure.

\input{theorem_transition_relative_interestingness_measure}
\input{proof_transition_relative_interestingness_measure}

All in all by sequentially applying \Cref{lemma:loss_to_interestingness} (respectively \Cref{lemma:score_to_interestingness}) and then \Cref{lemma:transition_to_relative_interestingness} to a soft classifier performance measure upper (respectively lower) bound for a performance measure $m$ we obtain a tight optimistic estimate for the relative soft classifier loss (respectively score) interestingness measure based on $m$.
This proving scheme is used to prove tight optimistic estimates for the relative ARL, ROC AUC and PR AUC based interestingness measures.

\subsection{Average Ranking Loss} \label{sec:arl_optimistic_estimate_proofs}

\input{theorem_upper_bound_ARL}
\input{proof_upper_bound_ARL}

\tightoearl*
\input{proof_tight_optimistic_estimate_ARL}

\subsection{Area Under the ROC Curve} \label{sec:rocauc_optimistic_estimate_proofs}

\tightboundrocauc*
\input{proof_lower_bound_ROCAUC}

\tightoerocauc*
\input{proof_tight_optimistic_estimate_ROCAUC}

\subsection{Area Under the Precision-Recall Curve} \label{sec:prauc_optimistic_estimate_proofs}

\input{figure_prc_worst_approximation}

\tightboundprauc*
\input{proof_lower_bound_PRAUC}

\tightoeprauc*
\input{proof_tight_optimistic_estimate_PRAUC}

\tightboundweighting*
\input{proof_upper_bound_weighting}

\newpage
\section{Experimental Results}
\label{apx:experimental_results}

\FloatBarrier
\subsection{Dataset statistics.}

\begin{table}[ht]
    \centering
    \caption{Dataset metadata}
    \input{table_datasets_overview}
    \label{tab:datasets_overview}
\end{table}
\vfill

\FloatBarrier
\subsection{Quality Parameters of Results Across Multiple Datasets}
\label{apx:sec:quality_parameters}

\begin{table}[ht]
    \centering
    \caption{
    Overview of result set properties for combinations of datasets and performance measures for an \texttt{xgboost} classifier.
    (Continuation of \Cref{tab:generalizability_metrics_overview_part_1} in the main text)
    }
    \input{table_result_quality_summary_depth_4_manually_rearranged_appendix}
    \label{tab:generalizability_metrics_overview}
\end{table}
\vfill

\begin{table}[ht]
    \centering
    \caption{
    Overview of result set properties for combinations of datasets and performance measures for an \texttt{xgboost} classifier.
    No cover size weighting, class balance weighting, generalization awareness or significance filtering was applied.
    Subgroups were required to cover at least 20 instances and the search stopped at depth 4.
    Here, the search returns the top-100 subgroups, which were filtered for significance, and finally, the top-5 subgroups are returned.
    }
    \input{table_result_quality_summary_depth_4_xgboost_XGBClassifier_0_0_False_True}
    \label{tab:generalizability_metrics_overview_filtered}
\end{table}

\FloatBarrier
\subsection{Case Study: Adult}
\label{apx:sec:case_study_adult}

The following are full result sets of the searches on the \texttt{Adult} dataset, that are also reported in \Cref{tab:generalizability_metrics_overview_part_1} in the main text in terms of aggregated statistics.
Note in particular the amount of overlap between patterns in the result sets from the base SubROC setting compared to the subgroups in the results of the full SubROC framework.
Moreover values of the main text table are reflected in the extreme performance results, cover sizes and negative class ratios.

\begin{sidewaystable}[p]
    \centering
    \caption{Result set of basic SubROC setting with ARL.}
    \input{table_case_study_result_set_table_0_0_average_ranking_loss_False}
    \label{tab:base_arl_result_set}
\end{sidewaystable}

\begin{sidewaystable}[p]
    \centering
    \caption{Result set of basic SubROC setting with \gls{roc} \gls{auc}.}
    \input{table_case_study_result_set_table_0_0_sklearn.metrics.roc_auc_score_False}
    \label{tab:base_roc_auc_result_set}
\end{sidewaystable}

\begin{sidewaystable}[p]
    \centering
    \caption{Result set of basic SubROC setting with \gls{pr} \gls{auc}.}
    \input{table_case_study_result_set_table_0_0_prc_auc_score_False}
    \label{tab:base_pr_auc_result_set}
\end{sidewaystable}

\begin{sidewaystable}[p]
    \centering
    \caption{Result set of full SubROC framework with ARL.}
    \input{table_case_study_filtered_result_set_table_1_1_average_ranking_loss_True}
    \label{tab:full_approach_arl_result_set}
\end{sidewaystable}

\begin{sidewaystable}[p]
    \centering
    \caption{Result set of full SubROC framework with \gls{roc} \gls{auc}. (Extended version of \Cref{tab:example_result_set} in the main text)}
    \input{table_case_study_filtered_result_set_table_1_1_sklearn.metrics.roc_auc_score_True}
    \label{tab:full_approach_roc_auc_result_set}
\end{sidewaystable}

\begin{sidewaystable}[p]
    \centering
    \caption{Result set of full SubROC framework with \gls{pr} \gls{auc}.}
    \input{table_case_study_filtered_result_set_table_1_1_prc_auc_score_True}
    \label{tab:full_approach_pr_auc_result_set}
\end{sidewaystable}

\FloatBarrier

\subsection{Searching for an Injected Subgroup}
\label{apx:sec:subgroup_injection}

For details on the experimental setup, see \Cref{subsec:exp_subgroup_injection} in the main text.

\input{figure_injection_experiment_combined_arl}

\input{figure_injection_experiment_combined_pr_auc}

\FloatBarrier

\subsection{Optimistic Estimates}
\label{apx:sec:optimistic_estimates}

\begin{table}[H]
    \centering
    \caption{
        Speedups of median runtimes using optimistic estimates. (Continuation of \Cref{tab:time_speedups} in the main text)
    }
    \input{table_optimistic_estimates_depth_4_time_median_speedup_table_suppl_split}
    \label{tab:time_speedups_appendix}
\end{table}

\newpage

\begin{table}[H]
    \centering
    \caption{
    Median runtimes in seconds for ARL without optimistic estimate pruning, with equal weighting parameters $\alpha$ and $\beta$ in $\{0, 0.1, 0.3, 1\}$.
    Subgroups were required to cover at least 20 instances and the search stopped at depth 4.
    10 measurements were performed for each configuration, except for the dataset \enquote{UCI Census-Income (KDD)} for which it were 3 measurements.
    }
    \input{table_optimistic_estimates_depth_4_time_median_table_average_ranking_loss}
    \label{tab:runtimes_table_arl_depth_4}
\end{table}

\begin{table}[H]
    \centering
    \caption{
    Median runtimes in seconds for \gls{roc} \gls{auc} without optimistic estimate pruning, with equal weighting parameters $\alpha$ and $\beta$ in $\{0, 0.1, 0.3, 1\}$.
    Subgroups were required to cover at least 20 instances and the search stopped at depth 4.
    10 measurements were performed for each configuration, except for the dataset \enquote{UCI Census-Income (KDD)} for which it were 3 measurements.
    }
    \input{table_optimistic_estimates_depth_4_time_median_table_roc_auc_score}
    \label{tab:runtimes_table_roc_auc_depth_4}
\end{table}

\begin{table}[H]
    \centering
    \caption{
    Median runtimes in seconds for \gls{pr} \gls{auc} without optimistic estimate pruning, with equal weighting parameters $\alpha$ and $\beta$ in $\{0, 0.1, 0.3, 1\}$.
    Subgroups were required to cover at least 20 instances and the search stopped at depth 4.
    10 measurements were performed for each configuration, except for the dataset \enquote{UCI Census-Income (KDD)} for which it were 3 measurements.
    }
    \input{table_optimistic_estimates_depth_4_time_median_table_prc_auc_score}
    \label{tab:runtimes_table_pr_auc_depth_4}
\end{table}

\begin{table}[H]
    \centering
    \caption{
    Share of remaining evaluated subgroups in the same searches reported in \Cref{tab:runtimes_table_arl_depth_4,tab:runtimes_table_roc_auc_depth_4,tab:runtimes_table_pr_auc_depth_4}
    }
    \input{table_optimistic_estimates_depth_4_num_visited_speedup_table}
    \label{tab:num_visited_speedups_depth_4}
\end{table}

\begin{figure}[H]
    \centering
    \includegraphics[]{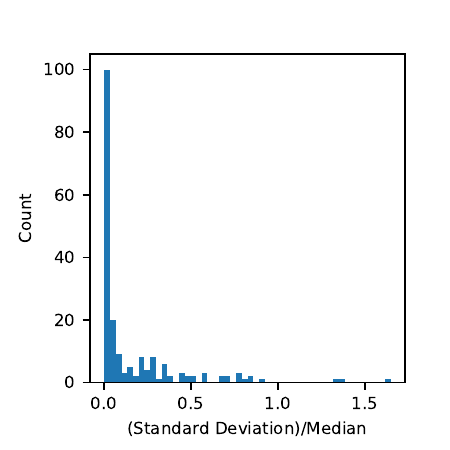}
    \caption{
    Histogram of standard deviations divided by medians for all runtimes of \Cref{tab:runtimes_table_arl_depth_4,tab:runtimes_table_roc_auc_depth_4,tab:runtimes_table_pr_auc_depth_4}.
    Most series of measurements did not experience noticeable runtime fluctuations compared to their median runtime.
    }
    \label{fig:runtimes_std_median_histogram_depth_4}
\end{figure}

%% file: theorem_transition_loss_interestingness_measure.tex
\begin{lemma}[Result Transition from Loss to Interestingness Measure]
    \label{lemma:loss_to_interestingness}
    Let $D = (\mathcal{I}, \mathcal{A})$ be a dataset.
    Let $m_l$ be a loss-type soft classifier performance measure with tight soft classifier performance measure upper bound $b_{m_l}$.
    Let $\varphi^{m_l}$ be a soft classifier loss interestingness measure based on $m_l$.
    A tight optimistic estimate for $\varphi^{m_l}$ is given by
    \begin{equation*}
        oe_{\varphi^{m_l}}(p) = b_{m_l}(sg(p)^{\{y, \hat{y}\}})
    \end{equation*}
    for any pattern $p$ that is defined on $D$.
\end{lemma}

%% file: proof_transition_loss_interestingness_measure.tex
\begin{proof}
    Let $p$ be a pattern that is defined on $D$.\\
    What has to be shown is that $oe_{\varphi^{m_l}}$ is an optimistic estimate of $\varphi^{m_l}$:
    \begin{equation} \label{eq:proof_loss_to_interestingness_1}
        \forall p' \supset p: \varphi^{m_l}(p') \leq oe_{\varphi^{m_l}}(p)
    \end{equation}
    and that this optimistic estimate is tight:
    \begin{equation} \label{eq:proof_loss_to_interestingness_2}
        \exists sg' \subseteq sg(p): sg(p') = sg' \wedge \varphi^{m_l}(p') = oe_{\varphi^{m_l}}(p).
    \end{equation}

    \textit{Proof for \Cref{eq:proof_loss_to_interestingness_1}}\\
    By definition of $\varphi^{m_l}$ and $oe_{\varphi^{m_l}}$ this equation is equivalent to
    \begin{equation} \label{eq:proof_loss_to_interestingness_3}
        \forall p' \supset p: m_l(sg(p')^{\{y, \hat{y}\}}) \leq b_{m_l}(sg(p)^{\{y, \hat{y}\}}).
    \end{equation}
    The right side of this equation is constant since $D$ and $p$ are fixed.\\
    \cite{lemmerich_novel_2014} states that for all two rules $p_{gen}$ and $p_{spec}$ the following statement holds
    \begin{equation*}
        p_{gen} \subset p_{spec} \Rightarrow sg(p_{gen}) \supseteq sg(p_{spec}).
    \end{equation*}
    It follows
    \begin{equation*}
        \forall p' \supset p: sg(p') \subseteq sg(p).
    \end{equation*}
    From the definition of databases follows
    \begin{equation} \label{eq:proof_loss_to_interestingness_4}
        \forall p' \supset p: sg(p')^{\{y, \hat{y}\}} \subseteq sg(p)^{\{y, \hat{y}\}}.
    \end{equation}
    Because $b_{m_l}$ is defined to be a soft classifier performance measure upper bound of $m_l$ we know that
    \begin{equation} \label{eq:proof_loss_to_interestingness_5}
        \forall C' \subseteq C: m_l(C') \leq b_{m_l}(C)
    \end{equation}
    holds for any $C = [c_1, \dots, c_n]$ (multiset of instance values $c_i \in \{0, 1\} \times \mathbb{R}$).\\
    \Cref{eq:proof_loss_to_interestingness_3} follows from \Cref{eq:proof_loss_to_interestingness_4} and \Cref{eq:proof_loss_to_interestingness_5}, concluding the proof for \Cref{eq:proof_loss_to_interestingness_1}.

    \textit{Proof for \Cref{eq:proof_loss_to_interestingness_2}}\\
    By definition of $\varphi^{m_l}$ and $oe_{\varphi^{m_l}}$ this equation is equivalent to
    \begin{equation*}
        \exists sg' \subseteq sg(p): sg(p') = sg' \wedge m_l(sg(p')^{\{y, \hat{y}\}}) = b_{m_l}(sg(p)^{\{y, \hat{y}\}}).
    \end{equation*}
    The equivalence of $sg(p')$ and $sg'$ as well as the triviality of $sg' = sg'$ imply that the following equation is equivalent:
    \begin{equation} \label{eq:proof_loss_to_interestingness_6}
        \exists sg' \subseteq sg(p): m_l(sg'^{\{y, \hat{y}\}}) = b_{m_l}(sg(p)^{\{y, \hat{y}\}}).
    \end{equation}
    Because $b_{m_l}$ is defined to be a tight soft classifier performance measure upper bound of $m_l$ we know that
    \begin{equation*}
        \exists C' \subseteq C: m_l(C') = b_{m_l}(C)
    \end{equation*}
    for any $C = [c_1, \dots, c_n]$ (multiset of instance values $c_i \in \{0, 1\} \times \mathbb{R}$).\\
    By fixing $C = sg(p)^{\{y, \hat{y}\}}$ we get
    \begin{equation} \label{eq:proof_loss_to_interestingness_7}
        \exists C' \subseteq sg(p)^{\{y, \hat{y}\}}: m_l(C') = b_{m_l}(sg(p)^{\{y, \hat{y}\}}).
    \end{equation}
    The definition of databases states that for any dataset $D' = (\mathcal{I}', \mathcal{A}')$ every instance in $\mathcal{I}'$ creates exactly one element in $d(D')$.\\
    It follows
    \begin{equation} \label{eq:proof_loss_to_interestingness_8}
        \forall C' \subseteq sg(p)^{\{y, \hat{y}\}} \exists \mathcal{I}' \subseteq sg(p): C' = \mathcal{I'}^{\{y, \hat{y}\}}.
    \end{equation}
    From \Cref{eq:proof_loss_to_interestingness_7} and \Cref{eq:proof_loss_to_interestingness_8} follows
    \begin{equation} \label{eq:proof_loss_to_interestingness_9}
        \exists C' \subseteq sg(p)^{\{y, \hat{y}\}} \exists \mathcal{I}' \subseteq sg(p): C' = \mathcal{I'}^{\{y, \hat{y}\}} \wedge m_l(C') = b_{m_l}(sg(p)^{\{y, \hat{y}\}}).
    \end{equation}
    The indifference of the order in a series of the same quantifiers, the equality of $C'$ and $\mathcal{I'}^{\{y, \hat{y}\}}$ as well as simplification rules imply that the following is equivalent to \Cref{eq:proof_loss_to_interestingness_9}
    \begin{equation*}
        \exists \mathcal{I}' \subseteq sg(p): m_l(\mathcal{I'}^{\{y, \hat{y}\}}) = b_{m_l}(sg(p)^{\{y, \hat{y}\}}).
    \end{equation*}
    This is trivially equivalent to \Cref{eq:proof_loss_to_interestingness_6}, concluding the proof for \Cref{eq:proof_loss_to_interestingness_2}.
\end{proof}

%% file: theorem_transition_score_interestingness_measure.tex
\begin{lemma}[Result Transition from Score to Interestingness Measure]
    \label{lemma:score_to_interestingness}
    Let $D = (\mathcal{I}, \mathcal{A})$ be a dataset.
    Let $m_s$ be a score-type soft classifier performance measure with tight soft classifier performance measure lower bound $b_{m_s}$.
    Let $\varphi^{m_s}$ be a soft classifier score interestingness measure based on $m_s$.
    A tight optimistic estimate for $\varphi^{m_s}$ is given by
    \begin{equation*}
        oe_{\varphi^{m_s}}(p) = -b_{m_s}(sg(p)^{\{y, \hat{y}\}})
    \end{equation*}
    for any pattern $p$ that is defined on $D$.
\end{lemma}

%% file: proof_transition_score_interestingness_measure.tex
\begin{proof}
    Let $m_s'(C) = -m_s(C)$ and $b_{m_s'}(C) = -b_{m_s}(C)$ for any $C = [c_1, \dots, c_n]$ (multiset of instance values $c_i \in \{0, 1\} \times \mathbb{R}$).
    $m_s'$ is a loss-type soft classifier performance measure by definition and $b_{m_s'}$ is a tight soft classifier performance measure upper bound for $m_s'$.\\
    It follows from \Cref{lemma:loss_to_interestingness} that the soft classifier loss interestingness measure $\varphi^{m_s'}$ based on $m_s'$ has the following tight optimistic estimate
    \begin{equation*}
        oe_{\varphi^{m_s'}}(p) = b_{m_s'}(sg(p)^{\{y, \hat{y}\}})
    \end{equation*}
    for any pattern $p$ that is defined on $D$.\\
    By definition of $\varphi^{m_s}$ and $\varphi^{m_s'}$ we know that
    \begin{equation*}
        \varphi^{m_s}(p) = -m_s(sg(p)^{\{y, \hat{y}\}}) = \varphi^{m_s'}(p)
    \end{equation*}
    for any pattern $p$ that is defined on $D$.\\
    It follows that $oe_{\varphi^{m_s'}}$ is also a tight optimistic estimate for $\varphi^{m_s}$.\\
    $oe_{\varphi^{m_s'}}$ and $oe_{\varphi^{m_s}}$ define the same function.
\end{proof}

%% file: theorem_transition_relative_interestingness_measure.tex
\begin{lemma}[Result Transition to Relative Interestingness Measure]
    \label{lemma:transition_to_relative_interestingness}
    Let $D = (\mathcal{I}, \mathcal{A})$ be a dataset.
    Let $\varphi$ be an interestingness measure with tight optimistic estimate $oe_{\varphi}$.
    Let $c \in \mathbb{R}$ be a constant.
    A tight optimistic estimate for the interestingness measure $\varphi'(p) = \varphi(p) + c$ is given by
    \begin{equation*}
        oe_{\varphi'}(p) = oe_{\varphi}(p) + c
    \end{equation*}
    for any pattern $p$ that is defined on $D$.
\end{lemma}

%% file: proof_transition_relative_interestingness_measure.tex
\begin{proof}
    Let $p$ be a pattern that is defined on $D$.\\
    What has to be shown is that $oe_{\varphi'}$ is an optimistic estimate of $\varphi'$:
    \begin{equation} \label{eq:proof_transition_to_relative_interestingness_1}
        \forall p' \supset p: \varphi'(p') \leq oe_{\varphi'}(p)
    \end{equation}
    and that this optimistic estimate is tight:
    \begin{equation} \label{eq:proof_transition_to_relative_interestingness_2}
        \exists sg' \subseteq sg(p): sg(p') = sg' \wedge \varphi'(p') = oe_{\varphi'}(p).
    \end{equation}

    \textit{Proof for \Cref{eq:proof_transition_to_relative_interestingness_1}}\\
    By definition of $\varphi'$ and $oe_{\varphi'}$ this equation is equivalent to
    \begin{equation} \label{eq:proof_transition_to_relative_interestingness_3}
        \forall p' \supset p: \varphi(p') + c \leq oe_{\varphi}(p) + c.
    \end{equation}
    Because $oe_{\varphi}$ is defined to be an optimistic estimate of $\varphi$ we know that
    \begin{equation} \label{eq:proof_transition_to_relative_interestingness_4}
        \forall p' \supset p: \varphi(p') \leq oe_{\varphi}(p)
    \end{equation}
    \Cref{eq:proof_transition_to_relative_interestingness_3} is true if and only if \Cref{eq:proof_transition_to_relative_interestingness_4} is true, concluding the proof for \Cref{eq:proof_transition_to_relative_interestingness_1}.

    \textit{Proof for \Cref{eq:proof_transition_to_relative_interestingness_2}}\\
    By definition of $\varphi'$ and $oe_{\varphi'}$ this equation is equivalent to
    \begin{equation} \label{eq:proof_transition_to_relative_interestingness_5}
        \exists sg' \subseteq sg(p): sg(p') = sg' \wedge \varphi(p') + c = oe_{\varphi}(p) + c.
    \end{equation}
    Because $oe_{\varphi}$ is defined to be a tight optimistic estimate of $\varphi$ we know that
    \begin{equation} \label{eq:proof_transition_to_relative_interestingness_6}
        \exists sg' \subseteq sg(p): sg(p') = sg' \wedge \varphi(p') = oe_{\varphi}(p).
    \end{equation}
    \Cref{eq:proof_transition_to_relative_interestingness_5} is true if and only if \Cref{eq:proof_transition_to_relative_interestingness_6} is true, concluding the proof for \Cref{eq:proof_transition_to_relative_interestingness_2}.
\end{proof}

%% file: theorem_upper_bound_ARL.tex
\begin{lemma}[Tight Soft Classifier Performance Measure Upper Bound $b_{ARL}$]
    \label{lemma:tight_bound_arl}
    A tight soft classifier performance measure upper bound for the average ranking loss is given by the function
    \begin{equation*}
        b_{ARL}(C) = \max_{i \in \{1, \dots, |C|\} } \{RL_i(C)\}
    \end{equation*}
    for any $C = [c_1, \dots, c_n]$ (multiset of instance values $c_i \in \{0, 1\} \times \mathbb{R}$).
\end{lemma}

%% file: proof_upper_bound_ARL.tex
\begin{proof}
    Let $C = [c_1, \dots, c_n]$ be a multiset of instance values $c_i \in \{0, 1\} \times \mathbb{R}$.\\
    What has to be shown is that $b_{ARL}$ is a soft classifier performance measure upper bound for the average ranking loss:
    \begin{equation} \label{eq:proof_tight_bound_arl_1}
        \forall C' \subseteq C: ARL(C') \leq b_{ARL}(C)
    \end{equation}
    and that this soft classifier performance measure upper bound is tight:
    \begin{equation} \label{eq:proof_tight_bound_arl_2}
        \exists C' \subseteq C: ARL(C') = b_{ARL}(C).
    \end{equation}

    \textit{Proof for \Cref{eq:proof_tight_bound_arl_1}}\\
    By definition of the average ranking loss (see main text) and $b_{ARL}$ this equation is equivalent to
    \begin{equation} \label{eq:proof_tight_bound_arl_3}
        \forall C' \subseteq C: \frac{\sum_{i=1}^{|C'|}RL_i(C')}{\sum_{i=1}^{|C'|} \1_{\{1\}} (y'_i)} \leq \max_{i \in \{1, \dots, |C|\}} \{RL_i(C)\}.
    \end{equation}
    The ranking loss $RL_i$ for the instance at position $i$ in any given multiset of instance values $c_j \in \{0, 1\} \times \mathbb{R}$ defines a weighted sum of the negative instances with rank at least as high as $i$ with weights that do not change for sub-multisets of the input.
    It is trivial that the maximum ranking loss of instances in $C$ does not increase for sub-multisets $C'$ of $C$:
    \begin{equation} \label{eq:proof_tight_bound_arl_4}
        \forall C' \subseteq C: \max_{i \in \{1, \dots, |C'|\}} \{RL_i(C')\} \leq \max_{i \in \{1, \dots, |C|\}} \{RL_i(C)\}
    \end{equation}
    It follows that the following expression implies \Cref{eq:proof_tight_bound_arl_3}.
    \begin{equation} \label{eq:proof_tight_bound_arl_5}
        \forall C' \subseteq C: \frac{\sum_{i=1}^{|C'|}RL_i(C')}{\sum_{i=1}^{|C'|} \1_{\{1\}} (y'_i)} \leq \max_{i \in \{1, \dots, |C'|\}} \{RL_i(C')\}
    \end{equation}
    It remains to show that \Cref{eq:proof_tight_bound_arl_5} holds.\\
    It is trivial that the average of a multiset $[m_1, \dots, m_{|M|}] = M$ of real numbers $m_i \in \mathbb{R}$ is not greater than the maximum of $M$:
    \begin{equation} \label{eq:proof_tight_bound_arl_6}
        \frac{\sum_{i=1}^{|M|} m_i}{|M|} \leq \max M
    \end{equation}
    For any $C' = [c_1', \dots, c_n']$ (multiset of instance values $c_i' \in \{0, 1\} \times \mathbb{R}$) and corresponding multiset of ranking loss values of positive instances $M_{C'} = [m_i = RL_i(C') \mid i \in \{1, \dots, |C'|\} \wedge (y'_i, \hat{y}'_i) \in C' \wedge y'_i = 1]$ the $ARL$ of $C'$ is equal to the average of the values in $M_{C'}$:
    \begin{equation} \label{eq:proof_tight_bound_arl_7}
        \frac{\sum_{i=1}^{|C'|}RL_i(C')}{\sum_{i=1}^{|C'|} \1_{\{1\}} (y'_i)} = \frac{\sum_{i=1}^{|M_{C'}|} m_i}{|M_{C'}|}
    \end{equation}
    \Cref{eq:proof_tight_bound_arl_5} follows from \Cref{eq:proof_tight_bound_arl_6} and \Cref{eq:proof_tight_bound_arl_7}, concluding the proof for \Cref{eq:proof_tight_bound_arl_1}.

    \textit{Proof for \Cref{eq:proof_tight_bound_arl_2}}\\
    By definition of the average ranking loss and $b_{ARL}$ this equation is equivalent to
    \begin{equation} \label{eq:proof_tight_bound_arl_8}
        \exists C' \subseteq C: \frac{\sum_{i=1}^{|C'|}RL_i(C')}{\sum_{i=1}^{|C'|} \1_{\{1\}} (y'_i)} = \max_{i \in \{1, \dots, |C|\}} \{RL_i(C)\}.
    \end{equation}
    The ranking loss $RL_i$ for the instance at position $i$ in any given multiset of instance values $c_j \in \{0, 1\} \times \mathbb{R}$ defines a weighted sum of the negative instances with rank at least as high as $i$ with weights that do not change for sub-multisets of the input.\\
    Let $i = \arg \max_{i \in \{1, \dots, |C|\}} \{RL_i(C)\}$ and $C' = (C \circledast \{(y, \hat{y}) \in C \mid y = 0 \wedge \hat{y} \geq \hat{y}_i\}) \cup [c_i]$.
    $C'$ consists of the positive instance $c_i$ and exactly the elements of $C$ that define the non-zero summands in the definition of $PEN_i$ (see main text).
    By definition of the ranking loss, $C'$ demonstrates the existence of a sub-multiset of $C$ for which \Cref{eq:proof_tight_bound_arl_8} holds.\\
    This concludes the proof for \Cref{eq:proof_tight_bound_arl_2}.
\end{proof}

%% file: proof_tight_optimistic_estimate_ARL.tex
\begin{proof}
    Let $D = (\mathcal{I}, \mathcal{A})$ be a dataset.\\
    \Cref{lemma:tight_bound_arl} states that a tight soft classifier performance measure upper bound for the (loss-type) average ranking loss is given by the function
    \begin{equation*}
        b_{ARL}(C) = \max_{i \in \{1, \dots, |C|\} } \{RL_i(C)\}
    \end{equation*}
    for any $C = [c_1, \dots, c_n]$ (multiset of instance values $c_i \in \{0, 1\} \times \mathbb{R}$).\\
    It follows from \Cref{lemma:loss_to_interestingness} that a tight optimistic estimate for $\varphi^{asl}$ is given by
    \begin{equation*}
        oe_{\varphi^{asl}}(p) = \max_{i \in \{1, \dots, |sg(p)|\} } \{RL_i(sg(p)^{\{y, \hat{y}\}})\}
    \end{equation*}
    for any pattern $p$ that is defined on $D$.\\
    $c = - ARL(\mathcal{I}^{\{y, \hat{y}\}})$ is constant with respect to the fixed dataset $D$.
    It follows from \Cref{lemma:transition_to_relative_interestingness} that a tight optimistic estimate for $\varphi^{rasl}$ is given by
    \begin{equation*}
        oe_{\varphi^{rasl}}(p) = \max_{i \in \{1, \dots, |sg(p)|\} } \{RL_i(sg(p)^{\{y, \hat{y}\}})\} - ARL(\mathcal{I}^{\{y, \hat{y}\}})
    \end{equation*}
    for any pattern $p$ that is defined on $D$.\\
    This result is equivalent to the function in the theorem.
\end{proof}

%% file: proof_lower_bound_ROCAUC.tex
\begin{proof}
    Let $I$ be a set of instances such that $C = I^{\{y, \hat{y}\}} = [c_1, \dots, c_n]$ is a multiset of instance values $(y_i, \hat{y}_i) \in \{0, 1\} \times \mathbb{R}$ that contains at least one instance with $y_i = 1$ and one with $y_i = 0$ each.\\
    We show that $b_{ROCAUC}$ is a tight lower bound using an alternative signature:
    \begin{equation*}
        b_{ROCAUC}(C) =
        \begin{cases}
            1,& \text{if } \forall (y, \hat{y}), (y', \hat{y}') \in C: \\
            & (y < y' \rightarrow \hat{y} < \hat{y}') \wedge (y > y' \rightarrow \hat{y} > \hat{y}')\\
            \frac{1}{2},& \text{if } \forall (y, \hat{y}), (y', \hat{y}') \in C: \\
            & (y < y' \rightarrow \hat{y} \leq \hat{y}') \wedge (y > y' \rightarrow \hat{y} \geq \hat{y}')\\
            0,& \text{otherwise}
        \end{cases}
    \end{equation*} \\
    The three cases $b_{ROCAUC}(C) = 1$, $b_{ROCAUC}(C) = \frac{1}{2}$ and $b_{ROCAUC}(C) = 0$ are treated separately.

    The case $b_{ROCAUC}(C) = 1$ is present if and only if $C$ contains perfect predictions, meaning that all positive instances in $C$ have strictly greater predicted values than all negative instances in $C$.
    All sub-multisets $C'$ of $C$ such that $ROCAUC(C') \neq undefined$ must contain at least one positive and one negative instance.
    It follows that all eligible sub-multisets of $C$ also contain perfect predictions.
    The \gls{roc} \gls{auc} is equal to $1$ for perfect predictions.
    $b_{ROCAUC}$ consequently meets the criterion of a soft classifier performance measure lower bound for $ROCAUC$ on perfect predictions.
    The tightness criterion is also met since $ROCAUC(C) = 1$.
    
    The case $b_{ROCAUC}(C) = \frac{1}{2}$ is present if and only if $C$ contains perfect predictions with ties, meaning that predictions in $C$ are similar to the case $b_{ROCAUC}(C) = 1$ except for the negative instances with the highest predicted value (compared to other negative instances) and the positive instances with the lowest predicted value (compared to other positive instances), which have equal predicted values (i.e. they are tied).
    First we show that $b_{ROCAUC}$ satisfies the tightness criterion in this case and then that it is a soft classifier performance measure lower bound for \gls{roc} \gls{auc} in this case.

    \noindent
    The sub-multiset of $C$ that only contains ties has a \gls{roc} \gls{auc} of $\frac{1}{2}$.
    Therefore the tightness criterion is fulfilled in this case.
    
    \noindent
    From the definition of $b_{ROCAUC}$ we know that in this case
    \begin{equation} \label{eq:rocauc_proof_case_2_cond}
        \forall (y, \hat{y}), (y', \hat{y}') \in C: (y < y' \rightarrow \hat{y} \leq \hat{y}') \wedge (y > y' \rightarrow \hat{y} \geq \hat{y}')
    \end{equation}
    holds.
    This condition also holds for any sub-multiset of $C$, that contains at least one instance with $y_i = 1$ and one with $y_i = 0$ each (such that \gls{roc} \gls{auc} is defined), because of the universal quantifier.
    Let $C'$ be such a sub-multiset.
    We continue with a proof by contradiction, showing that the premise $ROCAUC(C') < \frac{1}{2}$ implies that \Cref{eq:rocauc_proof_case_2_cond} is not true.
    Firstly the diagonal \gls{roc} curve with only supporting points $(0, 0)$ and $(1, 1)$ has an \gls{auc} of $\frac{1}{2}$.
    Clearly any \gls{roc} curve with only supporting points $(0, 0)$, $(1, 1)$ and additional points above or on the diagonal \gls{roc} curve has an \gls{auc} of $\frac{1}{2} + x$ with $x \in \mathbb{R}_{\geq 0}$.
    Therefore $ROCAUC(C') < \frac{1}{2}$ implies that at least one supporting point of the \gls{roc} curve of $C'$ lies below the diagonal \gls{roc} curve.
    This is equivalently expressed by
    \begin{align*}
        & & &\exists (TP, FP) \in S(C'): &\frac{TP}{P(C')} &< \frac{FP}{N(C')} \\
        & \Leftrightarrow & &\exists t \in \{\hat{y} \mid (y, \hat{y}) \in C'\} \cup \{-\infty\}: &\frac{TP(C', t)}{P(C')} &< \frac{FP(C', t)}{N(C')}\\
        & \Leftrightarrow & &\exists t \in \{\hat{y} \mid (y, \hat{y}) \in C'\}: &\frac{1}{P(C')} \cdot |[(y, \hat{y}) &\in C' \mid y = 1 \wedge \hat{y} > t]|<\\
        & & & &\frac{1}{N(C')} \cdot |[(y, \hat{y}) &\in C' \mid y = 0 \wedge \hat{y} > t]|.
    \end{align*}
    Because the FDR and TPR are both limited to the range $[0,1]$, it follows
    \begin{align*}
        &\Rightarrow & \exists t \in \{\hat{y} \mid (y, \hat{y}) \in C'\}: TP(C', t) \neq P(C') \wedge FP(C', t) \neq 0.
    \end{align*}
    This implies that there exists a positive instance in $C'$ that is not counted in $TP(C', t)$ and that there exists a negative instance in $C'$ that is counted in $FP(C', t)$.
    Now from the definitions of $TP(C', t)$ and $FP(C', t)$ follows
    \begin{align*}
        \Rightarrow & \exists t \in \{\hat{y} \mid (y, \hat{y}) \in C'\}:\\
        &(\exists (y, \hat{y}) \in C': y = 1 \wedge \neg (\hat{y} > t)) \wedge (\exists (y, \hat{y}) \in C': y = 0 \wedge \hat{y} > t)\\
        \Leftrightarrow & \exists t \in \{\hat{y} \mid (y, \hat{y}) \in C'\} \exists (y, \hat{y}), (y', \hat{y}') \in C': y = 1 \wedge \hat{y} \leq t \wedge y' = 0 \wedge \hat{y}' > t\\
        \Leftrightarrow & \exists t \in \{\hat{y} \mid (y, \hat{y}) \in C'\} \exists (y, \hat{y}), (y', \hat{y}') \in C': y = 1 \wedge y' = 0 \wedge \hat{y} \leq t < \hat{y}'
    \end{align*}
    If there exist instances $(y, \hat{y})$ and $(y', \hat{y}')$ in $C'$ that fulfill this expression, then they also fulfill the following expressions:
    \begin{align*}
        &\Rightarrow & \exists (y, \hat{y}), (y', \hat{y}') \in C'&: & &y > y' \wedge \hat{y} < \hat{y}'\\
        &\Rightarrow & \exists (y, \hat{y}), (y', \hat{y}') \in C'&: & &(y < y' \wedge \hat{y} > \hat{y}') \vee (y > y' \wedge \hat{y} < \hat{y}')
    \end{align*}
    The negation of \Cref{eq:rocauc_proof_case_2_cond} is equivalent:
    \begin{align*}
        &\Leftrightarrow & \exists (y, \hat{y}), (y', \hat{y}') \in C'&: & &\neg (y \geq y' \wedge \hat{y} \leq \hat{y}') \vee \neg (y \leq y' \wedge \hat{y} \geq \hat{y}')\\
        &\Leftrightarrow & \exists (y, \hat{y}), (y', \hat{y}') \in C'&: & &\neg ((y < y' \rightarrow \hat{y} \leq \hat{y}') \wedge (y > y' \rightarrow \hat{y} \geq \hat{y}'))\\
        &\Leftrightarrow & \neg \forall (y, \hat{y}), (y', \hat{y}') \in C'&: & &(y < y' \rightarrow \hat{y} \leq \hat{y}') \wedge (y > y' \rightarrow \hat{y} \geq \hat{y}')
    \end{align*}
    Since in case $b_{ROCAUC}(C) = \frac{1}{2}$ we know that \Cref{eq:rocauc_proof_case_2_cond} holds, we can conclude that $ROCAUC(C'') \geq \frac{1}{2}$ holds for any sub-multiset of $C$, that contains at least one instance with $y_i = 1$ and one with $y_i = 0$ each (such that \gls{roc} \gls{auc} is defined).
    Therefore $b_{ROCAUC}$ is a soft classifier performance measure lower bound for \gls{roc} \gls{auc} in this case.
    
    The case $b_{ROCAUC}(C) = 0$ is present if and only if $C$ contains erroneous predictions, meaning that there exists at least two instances $(y, \hat{y})$ and $(y', \hat{y}')$ in $C$ such that $y = 0$, $y' = 1$ and $\hat{y} > \hat{y}'$.
    The sub-multiset of $C$ that only contains these two instances (once each) has an ROC AUC of $0$.
    This is the overall minimum of the ROC AUC.
    $b_{ROCAUC}$ consequently meets the criteria of a tight soft classifier performance measure lower bound for $ROCAUC$ in case $b_{ROCAUC}(C) = 0$.
\end{proof}

%% file: proof_tight_optimistic_estimate_ROCAUC.tex
\begin{proof}
    Let $D = (\mathcal{I}, \mathcal{A})$ be a dataset.\\
    \Cref{lemma:tight_bound_rocauc} states that a tight soft classifier performance measure lower bound for the (score-type) area under the ROC curve is given by $b_{ROCAUC}$.\\
    It follows from \Cref{lemma:score_to_interestingness} that a tight optimistic estimate for $\varphi^{ROCAUC}$ is given by
    \begin{equation*}
        oe_{\varphi^{ROCAUC}}(p) = -b_{ROCAUC}(sg(p)^{\{y, \hat{y}\}})
    \end{equation*}
    for any pattern $p$ that is defined on $D$.\\
    $c = ROCAUC(\mathcal{I}^{\{y, \hat{y}\}})$ is constant with respect to the fixed dataset $D$.
    It follows from \Cref{lemma:transition_to_relative_interestingness} that a tight optimistic estimate for $\varphi^{rROCAUC}$ is given by
    \begin{equation*}
        oe_{\varphi^{rROCAUC}} = ROCAUC(\mathcal{I}^{\{y, \hat{y}\}}) - b_{ROCAUC}(sg(p)^{\{y, \hat{y}\}})
    \end{equation*}
    for any pattern $p$ that is defined on $D$.\\
    This result is equivalent to the function in the theorem.
\end{proof}

%% file: figure_prc_worst_approximation.tex
\begin{figure}[t]
    \begin{center}
        \begin{tikzpicture}
            \begin{axis}[
                xmin=0, xmax=1,
                ymin=0, ymax=1,
                xlabel = recall,
                ylabel = precision,
                enlargelimits=false,
                width=5cm,
                height=5cm,
                legend pos=outer north east,
                legend cell align=left,
            ]
                \addplot [
                    domain=0:1,
                    samples=100,
                    smooth,
                    red,
                    dotted,
                    line width=1pt,
                ]
                {(x*1)/(x*1+1)};
                \addlegendentry{border}

                \addplot [
                    line width=1pt,
                ]
                coordinates {
                    (0,0)(1,0.5)
                };
                \addlegendentry{approximation 1}

                \addplot [
                    gray,
                    line width=1pt,
                ]
                coordinates {
                    (0.005,1)(0.005,0.005)(0.5,0.3333)(1,0.5)
                };
                \addlegendentry{approximation 2}
            \end{axis}
        \end{tikzpicture}
    \end{center}
    \caption{
        \textbf{Lower border of the achievable area in PR space and approximations for a positive class ratio $\frac{1}{2}$.}
        The curve named \enquote{border} is the actual border of the achievable area.
        The curves named \enquote{approximation 1} and \enquote{approximation 2} are the (linearly interpolated) PR curves for the multisets of instance values $[(1, 0), (1, 0), (0, 1), (0, 1)]$ and $[(1, 0), (1, 0.5), (0, 1), (0, 1)]$ respectively.
        Both include only erroneous predictions, meaning that all negative instances have greater predicted values than all positive instances.
    }
    \label{fig:prc_worst_approximation}
\end{figure}

%% file: proof_lower_bound_PRAUC.tex
\begin{proof}
    Let $I$ be a set of instances such that $C = I^{\{y, \hat{y}\}}$ is a multiset of instance values $(y, \hat{y}) \in \{0, 1\} \times \mathbb{R}$ that contains at least one instance with $y = 1$.
    Let
    $$(1, \hat{y}_{minpos}) = c_{minpos} = \arg \min_{(y, \hat{y}) \in [(y', \hat{y}') \in C \mid y' = 1]} \{\hat{y}\}$$
    and
    $$C_{worst} = [c_{minpos}] \cup (C \circledast \{(y, \hat{y}) \in C \mid y = 0\}).$$
    It holds $I_{worst}^{\{y, \hat{y}\}} = C_{worst}$, therefore $b_{PRAUC}(I) = PRAUC(C_{worst})$.
    Define for convenience $b_{PRAUC}(C) = PRAUC(C_{worst})$.\\
    The following three cases (\ref{eq:prauc_proof_case_1}, \ref{eq:prauc_proof_case_2} and \ref{eq:prauc_proof_case_3}) are treated separately:
    \begin{align}
        \forall (y, \hat{y}), (y', \hat{y}') \in C: (y < y' \rightarrow \hat{y} < \hat{y}') \wedge (y > y' \rightarrow \hat{y} > \hat{y}') \label{eq:prauc_proof_cond_1}\\
        \forall (y, \hat{y}), (y', \hat{y}') \in C: (y < y' \rightarrow \hat{y} \leq \hat{y}') \wedge (y > y' \rightarrow \hat{y} \geq \hat{y}') \label{eq:prauc_proof_cond_2}\\
        \text{\ref{eq:prauc_proof_cond_1}} \label{eq:prauc_proof_case_1}\\
        \text{\ref{eq:prauc_proof_cond_2}} \wedge \neg \text{\ref{eq:prauc_proof_cond_1}} \label{eq:prauc_proof_case_2}\\
        \neg \text{\ref{eq:prauc_proof_cond_2}} \label{eq:prauc_proof_case_3}
    \end{align}

    Assuming case \ref{eq:prauc_proof_case_1}.\\
    This case is present if and only if $C$ contains perfect predictions, meaning that all positive instances in $C$ have strictly greater predicted values than all negative instances in $C$.
    All sub-multisets $C'$ of $C$ such that $PRAUC(C') \neq undefined$ must contain at least one positive instance.
    It follows that all eligible sub-multisets of $C$ also contain perfect predictions or only positives.
    The \gls{pr} \gls{auc} is equal to $1$ for perfect predictions and for only positives.
    In particular, this also holds for $C_{worst}$.
    $b_{PRAUC}$ consequently meets the criterion of a soft classifier performance measure lower bound for $PRAUC$ on perfect predictions.
    The tightness criterion is also met since $PRAUC(C) = 1$.

    Assuming case \ref{eq:prauc_proof_case_2}.\\
    Define
    \begin{align*}
        T(C) &:= \{(-\infty, \hat{y}_1)\} \cup \{(\hat{y}_{i-1}, \hat{y}_i) \mid i \in \{2, \dots, |C|\} \wedge \hat{y}_{i-1} \neq \hat{y}_i\}\\
        T_{< \hat{y}_{minpos}}(C) &:= \{(t_{lo}, t_{hi}) \in T(C) \mid t_{hi} < \hat{y}_{minpos}\}\\
        T_{\geq \hat{y}_{minpos}}(C) &:= \{(t_{lo}, t_{hi}) \in T(C) \mid t_{hi} \geq \hat{y}_{minpos}\}
    \end{align*}
    From the definition of $\hat{y}_{minpos}$ follows
    \begin{align*}
        \forall t \in \{\hat{y} \mid (y, \hat{y}) \in C\}:\ &t < \hat{y}_{minpos}\\
        &\Rightarrow TP(C, t) = P(C)\\
        &\Rightarrow Rec(C, t) = \frac{P(C)}{P(C)} = 1,
    \end{align*}
    which implies
    \begin{align*}
    & \forall (t_{lo}, t_{hi}) \in T_{< \hat{y}_{minpos}}(C):\\
    & \quad \left| Rec(C, t_{hi}) - Rec(C, t_{lo}) \right| \cdot \frac{Prec(C, t_{hi}) + Prec(C, t_{lo})}{2}\\
    & \quad = \left| 1 - 1 \right| \cdot \frac{Prec(C, t_{hi}) + Prec(C, t_{lo})}{2} = 0.
    \end{align*}
    From the definitions of $PRAUC(C)$, $T(C)$, $T_{< \hat{y}_{minpos}}(C)$ and $T_{\geq \hat{y}_{minpos}}(C)$ follows
    \begin{align*}
        &PRAUC(C) \\
        & = \sum_{(t_{lo}, t_{hi}) \in T(C)} \left | Rec(C, t_{hi}) - Rec(C, t_{lo}) \right | \cdot \frac{Prec(C, t_{hi}) + Prec(C, t_{lo})}{2} \\
        & = \sum_{(t_{lo}, t_{hi}) \in T_{< \hat{y}_{minpos}}(C)} \left | Rec(C, t_{hi}) - Rec(C, t_{lo}) \right | \cdot \frac{Prec(C, t_{hi}) + Prec(C, t_{lo})}{2} \\
        & + \sum_{(t_{lo}, t_{hi}) \in T_{\geq \hat{y}_{minpos}}(C)} \left | Rec(C, t_{hi}) - Rec(C, t_{lo}) \right | \cdot \frac{Prec(C, t_{hi}) + Prec(C, t_{lo})}{2}.
    \end{align*}
    Therefore
    \begin{align}
        &PRAUC(C) \nonumber\\
        & = \sum_{(t_{lo}, t_{hi}) \in T_{\geq \hat{y}_{minpos}}(C)} \left | Rec(C, t_{hi}) - Rec(C, t_{lo}) \right | \cdot \frac{Prec(C, t_{hi}) + Prec(C, t_{lo})}{2} \nonumber\\
        & = PRAUC(C_{\hat{y} \geq \hat{y}_{minpos}}) \label{eq:prauc_proof_reduction_equality}
    \end{align}
    with
    \begin{align*}
        C_{\hat{y} \geq \hat{y}_{minpos}} :=&\ C \ominus (C \circledast \{(y, \hat{y}) \in C \mid \hat{y} < \hat{y}_{minpos}\})\\
        =&\ C \circledast \{(y, \hat{y}) \in C \mid \hat{y} \geq \hat{y}_{minpos}\}.
    \end{align*}
    Define $\mathcal{C}_{\hat{y} \subseteq \hat{y}_{minpos}}$ similarly for any multiset $\mathcal{C}$ of instance values $(y, \hat{y}) \in \{0, 1\} \times \mathbb{R}$ that contains at least one instance with $y = 1$.\\
    In particular, \Cref{eq:prauc_proof_reduction_equality} also holds for any sub-multiset $C'$ of $C$ with corresponding $C_{\hat{y} \geq \hat{y}_{minpos}}'$.
    Therefore
    $$\forall C' \subseteq C: PRAUC(C_{\hat{y} \geq \hat{y}_{minpos}}') = PRAUC(C')$$
    and
    $$b_{PRAUC}(C) = b_{PRAUC}(C_{\hat{y} \geq \hat{y}_{minpos}}).$$
    Therefore the lower bound condition on $C$
    $$\forall C' \subseteq C: PRAUC(C') \geq b_{PRAUC}(C)$$
    is true if and only if the lower bound condition holds on $C_{\hat{y} \geq \hat{y}_{minpos}}$
    $$\forall C' \subseteq C_{\hat{y} \geq \hat{y}_{minpos}}: PRAUC(C') \geq b_{PRAUC}(C_{\hat{y} \geq \hat{y}_{minpos}}).$$
    By definition of $C_{\hat{y} \geq \hat{y}_{minpos}}$
    $$\forall C' \subseteq C_{\hat{y} \geq \hat{y}_{minpos}}: C' = C_{\hat{y} \geq \hat{y}_{minpos}}'.$$
    Therefore the lower bound condition on $C_{\hat{y} \geq \hat{y}_{minpos}}$ is equivalent to
    $$\forall C' \subseteq C_{\hat{y} \geq \hat{y}_{minpos}}: PRAUC(C_{\hat{y} \geq \hat{y}_{minpos}}') \geq b_{PRAUC}(C_{\hat{y} \geq \hat{y}_{minpos}}).$$
    For any sub-multiset $C'$ of $C_{\hat{y} \geq \hat{y}_{minpos}}$ to have $$PRAUC(C_{\hat{y} \geq \hat{y}_{minpos}}') < PRAUC(C_{worst})$$ there must be at least one point on the \gls{pr} curve of $C_{\hat{y} \geq \hat{y}_{minpos}}'$ below that of $C_{worst}.$
    Consider any
    $$t \in \{\hat{y} \mid (y, \hat{y}) \in C_{\hat{y} \geq \hat{y}_{minpos}}' \wedge \hat{y} \geq \hat{y}_{minpos}\}.$$
    From the condition of case \ref{eq:prauc_proof_case_2} follows
    $$FP(C_{\hat{y} \geq \hat{y}_{minpos}}', t) = 0.$$
    This implies
    $$Prec(C_{\hat{y} \geq \hat{y}_{minpos}}', t) = 1.$$
    Therefore for any sub-multiset $C'$ of $C_{\hat{y} \geq \hat{y}_{minpos}}$ no supporting point on the \gls{pr} curve of $C_{\hat{y} \geq \hat{y}_{minpos}}'$ with such a threshold can lie below the \gls{pr} curve of $C_{worst}$.
    This also holds for the interpolated points between the supporting points because of the linear interpolation.\\
    It remains to show that the point with $t = -\infty$ for $C_{\hat{y} \geq \hat{y}_{minpos}}'$ is not below the \gls{pr} curve of $C_{worst}$.
    Clearly
    $$Rec(C_{\hat{y} \geq \hat{y}_{minpos}}', -\infty) = 1$$
    and
    $$Prec(C_{\hat{y} \geq \hat{y}_{minpos}}', -\infty) = \frac{P(C_{\hat{y} \geq \hat{y}_{minpos}}')}{P(C_{\hat{y} \geq \hat{y}_{minpos}}') + N(C_{\hat{y} \geq \hat{y}_{minpos}}')}.$$
    By definition
    $$P(C_{\hat{y} \geq \hat{y}_{minpos}}') \geq 1 = P(C_{worst})$$
    and
    $$N(C_{\hat{y} \geq \hat{y}_{minpos}}') \leq N(C_{\hat{y} \geq \hat{y}_{minpos}}) = N(C_{worst}).$$
    Therefore
    $$Prec(C_{\hat{y} \geq \hat{y}_{minpos}}', -\infty) \geq Prec(C_{worst}, -\infty).$$
    In conclusion, for this case we know that all points on the \gls{pr} curve of any sub-multiset of $C$ must either not lie below the \gls{pr} curve of $C_{worst}$ or have no influence on the \gls{auc}.
    Therefore the lower bound condition on $C$ is fulfilled by $b_{PRAUC}$ in this case.

    Assuming case \ref{eq:prauc_proof_case_3}.\\
    We show that no $C' \subseteq C$ produces a \gls{pr} point below the \gls{pr} curve of $C_{worst}$, which implies
    $$\nexists C' \subseteq C: PRAUC(C') < b_{PRAUC}(C_{\hat{y} \geq \hat{y}_{minpos}}).$$
    We know that $PRAUC(C) = PRAUC(C_{\hat{y} \geq \hat{y}_{minpos}})$, so it suffices to show
    \begin{equation}
        \nexists C' \subseteq C_{\hat{y} \geq \hat{y}_{minpos}}: PRAUC(C_{\hat{y} \geq \hat{y}_{minpos}}') < b_{PRAUC}(C_{\hat{y} \geq \hat{y}_{minpos}}).\label{eq:prauc_proof_case_3_target}
    \end{equation}
    In this case (with $C = C_{\hat{y} \geq \hat{y}_{minpos}}$) the supporting points of the \gls{pr} curve of $C_{worst}$ are
    \begin{align*}
        &\left\{ (0, 1), (0, 0), \left( 1, \frac{P(C_{worst})}{P(C_{worst}) + N(C_{worst})} \right) \right\}\\
        &= \left\{ (0, 1), (0, 0), \left( 1, \frac{1}{1 + N(C_{worst})} \right) \right\}.
    \end{align*}
    Let
    \begin{align*}
        \mathfrak{P}_{worst}' &= \left\{ (0, 0), \left( 1, \frac{P(C_{worst})}{P(C_{worst}) + N(C_{worst})} \right) \right\}\\
        &= \left\{ (0, 0), \left( 1, \frac{1}{1 + N(C_{worst})} \right) \right\}.
    \end{align*}
    Clearly the \gls{auc} between the points $(0, 1)$ and $(0, 0)$ is $0$, so
    $$PRAUC(C_{worst}) = AUC(\mathfrak{P}_{worst}').$$
    $\mathfrak{P}_{worst}'$ are the supporting points of the minimum \gls{pr} curve described in \cite{boyd_unachievable_2012}.
    As $\mathfrak{P}_{worst}'$ shows, the minimum \gls{pr} curve only depends on the positive class ratio.
    By definition of $C_{worst}$ there is no $C' \subseteq C_{\hat{y} \geq \hat{y}_{minpos}}$ such that the positive class ratio of $C_{\hat{y} \geq \hat{y}_{minpos}}'$ is strictly lower than that of $C_{worst}$.
    Also no point on the linear interpolation of $\mathfrak{P}_{worst}'$ is strictly above the correct \gls{pr} interpolation of $\mathfrak{P}_{worst}'$ as shown in the following.\\
    The points on the correct interpolation between the two points in $\mathfrak{P}_{worst}'$ according to \cite{boyd_unachievable_2012} are
    $$\left( r_i, \frac{\pi r_i}{(1 - \pi) + \pi r_i} \right)$$
    with $r_i = \frac{i}{1 + N(C_{worst})}$ and $\pi = \frac{1}{1 + N(C_{worst})}$ for $0 \leq i \leq 1$.\\
    The points on the linear interpolation between the two points in $\mathfrak{P}_{worst}'$ are
    $$\left( r_i, \pi r_i \right)$$
    with $r_i = \frac{i}{1 + N(C_{worst})}$ and $\pi = \frac{1}{1 + N(C_{worst})}$ for $0 \leq i \leq 1$.\\
    For the denominator of the second component in the points of the correct interpolation holds
    \begin{align*}
        & & &(1 - \pi) + \pi r_i\\
        = & & &\left( 1 - \frac{1}{1 + N(C_{worst})} \right) + \frac{1}{1 + N(C_{worst})} \cdot \frac{i}{1 + N(C_{worst})}\\
        = & & & \frac{1 + N(C_{worst}) - 1}{1 + N(C_{worst})} + \frac{1 \cdot i}{(1 + N(C_{worst}))^2}\\
        \leq & & & \frac{N(C_{worst})}{1 + N(C_{worst})} + \frac{1}{(1 + N(C_{worst}))^2}\\
        = & & & \frac{1 + N(C_{worst}) \cdot (1 + N(C_{worst}))}{(1 + N(C_{worst}))^2}\\
        = & & & \frac{1 + N(C_{worst}) + N(C_{worst})^2}{1 + 2 N(C_{worst}) + N(C_{worst})^2}\\
        \leq & & & 1.
    \end{align*}
    Therefore
    $$\frac{\pi r_i}{(1 - \pi) + \pi r_i} \geq \pi r_i.$$
    It follows that there is no $C' \subseteq C_{\hat{y} \geq \hat{y}_{minpos}}$ such that the \gls{pr} curve of $C_{\hat{y} \geq \hat{y}_{minpos}}'$ has a point that lies strictly below the \gls{pr} curve of $C_{worst}$.
    Thus \ref{eq:prauc_proof_case_3_target} is fulfilled and we can conclude that $b_{PRAUC}$ is a lower bound for the \gls{pr} \gls{auc} in case \ref{eq:prauc_proof_case_3}.
    It is also tight since $b_{PRAUC}(C) = PRAUC(C_{worst})$ and $C_{worst} \subseteq C$.
\end{proof}

%% file: proof_tight_optimistic_estimate_PRAUC.tex
\begin{proof}
    Let $D = (\mathcal{I}, \mathcal{A})$ be a dataset.\\
    \Cref{lemma:tight_bound_prauc} states that a tight soft classifier performance measure lower bound for the (score-type) area under the PR curve is given by $b_{PRAUC}$.\\
    It follows from \Cref{lemma:score_to_interestingness} that a tight optimistic estimate for $\varphi^{PRAUC}$ is given by
    \begin{equation*}
        oe_{\varphi^{PRAUC}}(p) = -b_{PRAUC}(sg(p)^{\{y, \hat{y}\}})
    \end{equation*}
    for any pattern $p$ that is defined on $D$.\\
    $c = PRAUC(\mathcal{I}^{\{y, \hat{y}\}})$ is constant with respect to the fixed dataset $D$.
    It follows from \Cref{lemma:transition_to_relative_interestingness} that a tight optimistic estimate for $\varphi^{rPRAUC}$ is given by
    \begin{equation*}
        oe_{\varphi^{rPRAUC}} = PRAUC(\mathcal{I}^{\{y, \hat{y}\}}) - b_{PRAUC}(sg(p)^{\{y, \hat{y}\}})
    \end{equation*}
    for any pattern $p$ that is defined on $D$.\\
    This result is equivalent to the function in the theorem.
\end{proof}

%% file: proof_upper_bound_weighting.tex
\begin{proof}
    Let $I$ be a set of instances such that $C = I^{\{y, \hat{y}\}}$ is a multiset of instance values $(y, \hat{y}) \in \{0, 1\} \times \R$ and $\alpha \leq \beta \in \R_{> 0}$ fixed real numbers.

    From the equalities $P(C) = |\mathcal{P}_I|$, $N(C) = |\mathcal{N}_I|$ and $cb(I) = cb(C)$, follows $w(I) = |C|^\alpha \cdot cb(C)^\beta$.
    Therefore we simply write $w(C)$. Similarly we use $b_w(C) = (2 \cdot \min{\{P(C), N(C)\}})^\alpha$ in place of $b_w(I)$.
    
    By definition
    $$|C| = P(C) + N(C).$$
    Therefore
    \begin{align*}
        w(C) &= (P(C) + N(C))^{\alpha} \cdot cb(C)^{\beta}.
    \end{align*}

    Assume $P(C) > N(C)$ and $N(C) > 0$ (Equivalent to the first case of the minimum in the $cb$ definition). \\
    By definition of $cb(C)$
    $$w(C) = (P(C) + N(C))^{\alpha} \cdot \left( \frac{N(C)}{P(C)} \right)^{\beta}.$$
    Let $c \in [(y, \hat{y}) \in C \mid y = 1]$ be any positive instance in $C$.\\
    We show that the weighting $w$ does not decrease when $c$ is removed from $C$.
    \begin{align}
        & & w(C \ominus [c]) &\geq w(C) \nonumber\\
        &\Rightarrow & (P(C) - 1 + N(C))^{\alpha} \left( \frac{N(C)}{P(C) - 1} \right)^{\beta} &\geq (P(C) + N(C))^{\alpha} \left( \frac{N(C)}{P(C)} \right)^{\beta} \nonumber\\
        &\Leftrightarrow & \frac{P(C) - 1 + N(C)}{P(C) + N(C)} &\geq \left( \frac{N(C)}{P(C)} \right)^{\frac{\beta}{\alpha}} \cdot \left( \frac{P(C) - 1}{N(C)} \right)^{\frac{\beta}{\alpha}} \nonumber\\
        &\Leftrightarrow & 1 &\geq \left( 1 - \frac{1}{P(C)} \right)^{\frac{\beta}{\alpha}} + \frac{1}{P(C) + N(C)} \label{eq:weight_oe_proof_intermediate_eq_case_1}
    \end{align}
    In case $\alpha = \beta$ follows
    \begin{align*}
        & & 1 &\geq \left( 1 - \frac{1}{P(C)} \right)^{1} + \frac{1}{P(C) + N(C)} \\
        &\Leftrightarrow & N(C) &\geq 0
    \end{align*}
    This is true because of the assumption $N(C) > 0$.\\
    From
    $$0 \leq 1 - \frac{1}{P(C)} \leq 1$$
    follows
    $$\left( 1 - \frac{1}{P(C)} \right)^{c} \leq 1 - \frac{1}{P(C)}$$
    for any constant $c \geq 1$.\\
    Therefore \Cref{eq:weight_oe_proof_intermediate_eq_case_1} holding in case $\alpha = \beta$ implies that it also holds in case $\alpha \leq \beta$.
    We conclude that the weighting $w$ does not decrease when $c$ is removed from $C$.\\
    Note that the case $P(C) = N(C) + 1$ works out here even though $cb(C \ominus [c])$ is not defined by the first case of the minimum in the $cb$ definition anymore.
    This follows from $P(C \ominus [c]) = N(C \ominus [c])$ in this case.\\
    Let $c' \in [(y, \hat{y}) \in C \mid y = 0]$ be any negative instance in $C$.
    Clearly $w(C \ominus [c']) \leq w(C)$, so removing any negative instance from $C$ does not increase the weighting $w$.\\
    In conclusion the weighting $w$ is maximized in the first case of the minimum in the $cb$ definition by sequentially removing a positive instance until the thereby created sub-multiset $C'$ of $C$ does not fulfill this case anymore, but fulfills $P(C') = N(C')$.\\
    For such a sub-multiset $C'$ of $C$ holds $cb(C') = 1$.
    Therefore
    $$w(C') = (P(C') + N(C'))^{\alpha}.$$
    From the definition of $C'$ follows
    $$P(C') = N(C') = \min \{P(C), N(C)\}.$$
    Therefore
    $$w(C') = (2 \cdot \min\{P(C), N(C)\})^{\alpha}.$$

    Assume $N(C) > P(C)$ and $P(C) > 0$ (Equal to the second case of the minimum in the $cb$ definition).\\
    It is analogous to the above to proof that removing any negative instance in $C$ does not decrease the weighting $w$ and that removing any positive instance in $C$ does not increase the weighting $w$.\\
    Therefore the weighting $w$ is maximized in the second case of the the minimum in the $cb$ definition by sequentially removing a negative instance until the thereby created sub-multiset $C'$ of $C$ does not fulfill $N(C') > P(C')$ anymore, but fulfills $N(C') = P(C')$.\\
    It is shown above that for such a sub-multiset $C'$ of $C$ the following holds:
    $$w(C') = (2 \cdot \min\{P(C), N(C)\})^{\alpha}$$

    Assume $N(C) = 0 \vee P(C) = 0$ (Equal to the case in the $cb$ definition returning $0$).\\
    In this case
    $$\forall C' \subseteq C: cb(C') = 0.$$
    Therefore
    $$\forall C' \subseteq C: w(C') = |C'|^{\alpha} \cdot 0^{\beta} = 0.$$
    $b_w$ is a tight upper bound in this case since it also returns $0$:
    $$b_w(C) = (2 \cdot \min\{P(C), N(C)\})^{\alpha} = (2 \cdot 0)^{\alpha} = 0$$
\end{proof}

%% file: table_datasets_overview.tex
\begin{tabular}{llrrr}
\toprule
Abbreviation & Name                         & \#Rows   & NCR    & \#Features \\ \midrule
Census & UCI Census-Income (KDD)      & 199523 & 0.0621 & 41 \\
Adult & OpenML Adult                 & 48842  & 0.2393 & 14 \\
Bank & UCI Bank Marketing           & 45211  & 0.8830 & 16 \\
Credit & UCI Credit Card Clients      & 30000  & 0.7788 & 23 \\
Mushroom & UCI Mushroom                 & 8124   & 0.5180 & 22 \\
Statlog & Statlog (German Credit Data) & 1000   & 0.3000 & 20 \\
Cancer& UCI Breast Cancer Wisconsin  & 699    & 0.6552 & 9 \\
Approval & UCI Credit Approval          & 690    & 0.5551 & 15 \\ \bottomrule
\end{tabular}

%% file: table_result_quality_summary_depth_4_manually_rearranged_appendix.tex
\begin{tabular}{ll?r|r?r|r?r|r?r|r?r|l}
 &  
    & \multicolumn{2}{c?}{\rotatebox{90}{\makecell[l]{Mean\\Exception-\\ality}}} 
    & \multicolumn{2}{c?}{\rotatebox{90}{\makecell[l]{Mean\\Cover\\Size}}} 
    & \multicolumn{2}{c?}{\rotatebox{90}{\makecell[l]{Mean\\NCR}}} 
    & \multicolumn{2}{c?}{\rotatebox{90}{\makecell[l]{Mean\\Pairwise\\IoU}}} 
    & \multicolumn{1}{c}{\rotatebox{90}{\makecell[l]{Significant/\\Top-5}}} 
    & \multicolumn{1}{c}{\rotatebox{90}{\makecell[l]{Filtered-5/\\Significant/\\Top-100}}} \\
Dataset & 
    \makecell{Metric} & 
    \makecell[c]{b} & \makecell[c]{f} &
    \makecell[c]{b} & \makecell[c]{f} &
    \makecell[c]{b} & \makecell[c]{f} &
    \makecell[c]{b} & \makecell[c]{f} &
    \makecell[c]{b} & \makecell[c]{f} \\
\midrule
\multirow[c]{3}{*}{\shortstack[l]{\textbf{Mushroom}\\n=8124\\NCR=0.52}} & 
 ROC & 
    0.00 & - & 
    203 & - & 
    0.76 & - & 
    0.28 & - & 
    0/5 & 0/0/100 \\
 \cline{2-12} & 
 PR & 
    0.00 & 0.00 & 
    35 & \textbf{57} & 
    0.17 & \textbf{0.49} & 
    0.14 & - & 
    0/5 & \textbf{1/1/100} \\
 \cline{2-12} & 
 ARL & 
    - & - & 
    - & - & 
    - & - & 
    - & - & 
    0/0 & 0/0/0 \\
\cline{1-12}
\multirow[c]{3}{*}{\shortstack[l]{\textbf{Statlog}\\n=1000\\NCR=0.30}} & 
ROC & 
    0.74 & - & 
    23 & - & 
    0.04 & - & 
    0.36 & - & 
    0/5 & 0/0/100 \\
 \cline{2-12}
 & PR & 
    0.50 & - & 
    21 & - & 
    0.63 & - & 
    0.95 & - & 
    0/5 & 0/0/100 \\
 \cline{2-12}
 & \multirow[c]{1}{*}{ARL} & 
    0.16 & - & 
    241 & - & 
    0.28 & - & 
    - & - & 
    0/1 & 0/0/1 \\
\cline{1-12}
\multirow[c]{3}{*}{\shortstack[l]{\textbf{Cancer}\\n=699\\NCR=0.66}} & 
ROC & 
    0.03 & - & 
    41 & - & 
    1.00 & - & 
    0.17 & - & 
    0/5 & 0/0/5 \\
 \cline{2-12}
 & PR & 
    0.55 & - & 
    78 & - & 
    1.00 & - & 
    0.80 & - & 
    0/5 & 0/0/7 \\
 \cline{2-12}
 & ARL & 
    0.18 & - & 
    83 & - & 
    1.00 & - & 
    - & - & 
    0/1 & 0/0/1 \\
\end{tabular}

%% file: table_result_quality_summary_depth_4_xgboost_XGBClassifier_0_0_False_True.tex
\begin{tabular}{llrrrrl}
\toprule
 &  & \rotatebox{90}{\makecell{Mean\\Exceptionality}} & \rotatebox{90}{\makecell{Mean\\Pairwise IoU}} & \rotatebox{90}{\makecell{Mean\\Cover Size}} & \rotatebox{90}{\makecell{Mean NCR}} & \rotatebox{90}{\makecell{Result Set Size\\(Top-5/\\Significant/\\Original)}} \\
Dataset & \makecell{Interestingness\\Measure} &  &  &  &  &  \\
\midrule
\multirow[c]{3}{*}{UCI Census-Income (KDD)} & ROC AUC & 0.95 & 0.46 & 33 & 0.03 & 5/11/100 \\
 & PR AUC & 0.84 & 0.60 & 20 & 0.94 & 5/13/100 \\
 & ARL & 46.72 & 0.99 & 8000 & 0.15 & 5/100/100 \\
\cline{1-7}
\multirow[c]{3}{*}{OpenML Adult} & ROC AUC & 0.80 & 0.16 & 25 & 0.06 & 5/14/100 \\
 & PR AUC & 0.93 & 0.21 & 26 & 0.95 & 5/14/100 \\
 & ARL & 147.66 & 0.85 & 4614 & 0.42 & 5/82/82 \\
\cline{1-7}
\multirow[c]{3}{*}{UCI Bank Marketing} & ROC AUC & - & - & - & - & 0/0/100 \\
 & PR AUC & - & - & - & - & 0/0/100 \\
 & ARL & - & - & - & - & 0/0/0 \\
\cline{1-7}
\multirow[c]{3}{*}{UCI Credit Card Clients} & ROC AUC & - & - & - & - & 0/0/100 \\
 & PR AUC & - & - & - & - & 0/0/100 \\
 & ARL & - & - & - & - & 0/0/0 \\
\cline{1-7}
\multirow[c]{3}{*}{UCI Mushroom} & ROC AUC & - & - & - & - & 0/0/100 \\
 & PR AUC & - & - & - & - & 0/0/100 \\
 & ARL & - & - & - & - & 0/0/0 \\
\cline{1-7}
\multirow[c]{3}{*}{Statlog (German Credit Data)} & ROC AUC & - & - & - & - & 0/0/100 \\
 & PR AUC & - & - & - & - & 0/0/100 \\
 & ARL & - & - & - & - & 0/0/1 \\
\cline{1-7}
\multirow[c]{3}{*}{UCI Breast Cancer Wisconsin} & ROC AUC & - & - & - & - & 0/0/8 \\
 & PR AUC & - & - & - & - & 0/0/15 \\
 & ARL & - & - & - & - & 0/0/1 \\
\cline{1-7}
\multirow[c]{3}{*}{UCI Credit Approval} & ROC AUC & 0.58 & - & 49 & 0.98 & 1/1/100 \\
 & PR AUC & 0.89 & 0.81 & 40 & 0.97 & 5/12/100 \\
 & ARL & 20.40 & 0.68 & 49 & 0.97 & 5/8/100 \\
\cline{1-7}
\bottomrule
\end{tabular}

%% file: table_case_study_result_set_table_0_0_average_ranking_loss_False.tex
\begin{tabular}{@{}r@{\hspace{2mm}}l@{\hspace{2mm}}r@{\hspace{2mm}}r@{\hspace{2mm}}r@{\hspace{2mm}}r@{\hspace{2mm}}r@{}}
\toprule
Interestingness & Pattern & ARL & PR AUC & ROC AUC & Cover & NCR \\
\midrule
\addlinespace[5pt] \makecell[r]{159.2829\\\ \\\ } & \makecell[l]{$\text{capital-gain}\in[0.0,114.0)\:\wedge$\\$\text{capital-loss}\in[0.0,213.0)\:\wedge$\\$\text{marital-status}=\text{Married-civ-spouse}$} & \makecell[r]{372.3026\\\ \\\ } & \makecell[r]{0.8409\\\ \\\ } & \makecell[r]{0.7780\\\ \\\ } & \makecell[r]{4309\\\ \\\ } & \makecell[r]{0.3892\\\ \\\ } \\
\addlinespace[5pt] \makecell[r]{157.6086\\\ } & \makecell[l]{$\text{capital-gain}\in[0.0,114.0)\:\wedge$\\$\text{marital-status}=\text{Married-civ-spouse}$} & \makecell[r]{370.6283\\\ } & \makecell[r]{0.8473\\\ } & \makecell[r]{0.8053\\\ } & \makecell[r]{4628\\\ } & \makecell[r]{0.4114\\\ } \\
\addlinespace[5pt] \makecell[r]{142.0611\\\ } & \makecell[l]{$\text{capital-loss}\in[0.0,213.0)\:\wedge$\\$\text{marital-status}=\text{Married-civ-spouse}$} & \makecell[r]{355.0808\\\ } & \makecell[r]{0.8611\\\ } & \makecell[r]{0.8346\\\ } & \makecell[r]{4961\\\ } & \makecell[r]{0.4328\\\ } \\
\addlinespace[5pt] \makecell[r]{141.2313} & \makecell[l]{$\text{marital-status}=\text{Married-civ-spouse}$} & \makecell[r]{354.2510} & \makecell[r]{0.8660} & \makecell[r]{0.8508} & \makecell[r]{5280} & \makecell[r]{0.4496} \\
\addlinespace[5pt] \makecell[r]{138.0929\\\ \\\ \\\ } & \makecell[l]{$\text{capital-gain}\in[0.0,114.0)\:\wedge$\\$\text{capital-loss}\in[0.0,213.0)\:\wedge$\\$\text{marital-status}=\text{Married-civ-spouse}\:\wedge$\\$\text{native-country}=\text{United-States}$} & \makecell[r]{351.1126\\\ \\\ \\\ } & \makecell[r]{0.8275\\\ \\\ \\\ } & \makecell[r]{0.7748\\\ \\\ \\\ } & \makecell[r]{3894\\\ \\\ \\\ } & \makecell[r]{0.4004\\\ \\\ \\\ } \\
\midrule
\addlinespace[5pt] \makecell[r]{0.0000} & \makecell[l]{$\emptyset$} & \makecell[r]{213.0197} & \makecell[r]{0.9728} & \makecell[r]{0.9241} & \makecell[r]{11305} & \makecell[r]{0.2482} \\
\bottomrule
\end{tabular}

%% file: table_case_study_result_set_table_0_0_sklearn.metrics.roc_auc_score_False.tex
\begin{tabular}{@{}r@{\hspace{2mm}}l@{\hspace{2mm}}r@{\hspace{2mm}}r@{\hspace{2mm}}r@{\hspace{2mm}}r@{\hspace{2mm}}r@{}}
\toprule
Interestingness & Pattern & ARL & PR AUC & ROC AUC & Cover & NCR \\
\midrule
\addlinespace[5pt] \makecell[r]{0.9241\\\ \\\ \\\ } & \makecell[l]{$\text{education}=\text{Some-college}\:\wedge$\\$\text{fnlwgt}\in[107160.0,132053.0)\:\wedge$\\$\text{marital-status}=\text{Divorced}\:\wedge$\\$\text{workclass}=\text{Private}$} & \makecell[r]{1.0000\\\ \\\ \\\ } & \makecell[r]{0.8439\\\ \\\ \\\ } & \makecell[r]{0.0000\\\ \\\ \\\ } & \makecell[r]{21\\\ \\\ \\\ } & \makecell[r]{0.0476\\\ \\\ \\\ } \\
\addlinespace[5pt] \makecell[r]{0.9241\\\ \\\ \\\ } & \makecell[l]{$\text{education}=\text{Some-college}\:\wedge$\\$\text{fnlwgt}\in[178811.0,196338.0)\:\wedge$\\$\text{relationship}=\text{Not-in-family}\:\wedge$\\$\text{sex}=\text{Male}$} & \makecell[r]{1.0000\\\ \\\ \\\ } & \makecell[r]{0.8857\\\ \\\ \\\ } & \makecell[r]{0.0000\\\ \\\ \\\ } & \makecell[r]{32\\\ \\\ \\\ } & \makecell[r]{0.0312\\\ \\\ \\\ } \\
\addlinespace[5pt] \makecell[r]{0.9241\\\ \\\ \\\ } & \makecell[l]{$\text{capital-gain}\in[0.0,114.0)\:\wedge$\\$\text{education}=\text{Some-college}\:\wedge$\\$\text{fnlwgt}\in[107160.0,132053.0)\:\wedge$\\$\text{marital-status}=\text{Divorced}$} & \makecell[r]{1.0000\\\ \\\ \\\ } & \makecell[r]{0.8770\\\ \\\ \\\ } & \makecell[r]{0.0000\\\ \\\ \\\ } & \makecell[r]{29\\\ \\\ \\\ } & \makecell[r]{0.0345\\\ \\\ \\\ } \\
\addlinespace[5pt] \makecell[r]{0.9241\\\ \\\ \\\ } & \makecell[l]{$\text{education-num}\in[10,11)\:\wedge$\\$\text{fnlwgt}\in[107160.0,132053.0)\:\wedge$\\$\text{marital-status}=\text{Divorced}\:\wedge$\\$\text{race}=\text{White}$} & \makecell[r]{1.0000\\\ \\\ \\\ } & \makecell[r]{0.8703\\\ \\\ \\\ } & \makecell[r]{0.0000\\\ \\\ \\\ } & \makecell[r]{27\\\ \\\ \\\ } & \makecell[r]{0.0370\\\ \\\ \\\ } \\
\addlinespace[5pt] \makecell[r]{0.9241\\\ \\\ \\\ } & \makecell[l]{$\text{education-num}\in[10,11)\:\wedge$\\$\text{fnlwgt}\in[107160.0,132053.0)\:\wedge$\\$\text{marital-status}=\text{Divorced}\:\wedge$\\$\text{workclass}=\text{Private}$} & \makecell[r]{1.0000\\\ \\\ \\\ } & \makecell[r]{0.8439\\\ \\\ \\\ } & \makecell[r]{0.0000\\\ \\\ \\\ } & \makecell[r]{21\\\ \\\ \\\ } & \makecell[r]{0.0476\\\ \\\ \\\ } \\
\midrule
\addlinespace[5pt] \makecell[r]{0.0000} & \makecell[l]{$\emptyset$} & \makecell[r]{213.0197} & \makecell[r]{0.9728} & \makecell[r]{0.9241} & \makecell[r]{11305} & \makecell[r]{0.2482} \\
\bottomrule
\end{tabular}

%% file: table_case_study_result_set_table_0_0_prc_auc_score_False.tex
\begin{tabular}{@{}r@{\hspace{2mm}}l@{\hspace{2mm}}r@{\hspace{2mm}}r@{\hspace{2mm}}r@{\hspace{2mm}}r@{\hspace{2mm}}r@{}}
\toprule
Interestingness & Pattern & ARL & PR AUC & ROC AUC & Cover & NCR \\
\midrule
\addlinespace[5pt] \makecell[r]{0.9395\\\ \\\ \\\ } & \makecell[l]{$\text{education}=\text{Bachelors}\:\wedge$\\$\text{hours-per-week}\in[44,50)\:\wedge$\\$\text{occupation}=\text{Prof-specialty}\:\wedge$\\$\text{relationship}=\text{Husband}$} & \makecell[r]{14.0000\\\ \\\ \\\ } & \makecell[r]{0.0333\\\ \\\ \\\ } & \makecell[r]{0.4400\\\ \\\ \\\ } & \makecell[r]{26\\\ \\\ \\\ } & \makecell[r]{0.9615\\\ \\\ \\\ } \\
\addlinespace[5pt] \makecell[r]{0.9395\\\ \\\ \\\ } & \makecell[l]{$\text{education-num}\in[13,14)\:\wedge$\\$\text{hours-per-week}\in[44,50)\:\wedge$\\$\text{occupation}=\text{Prof-specialty}\:\wedge$\\$\text{relationship}=\text{Husband}$} & \makecell[r]{14.0000\\\ \\\ \\\ } & \makecell[r]{0.0333\\\ \\\ \\\ } & \makecell[r]{0.4400\\\ \\\ \\\ } & \makecell[r]{26\\\ \\\ \\\ } & \makecell[r]{0.9615\\\ \\\ \\\ } \\
\addlinespace[5pt] \makecell[r]{0.9371\\\ \\\ \\\ } & \makecell[l]{$\text{education-num}\in[14,15)\:\wedge$\\$\text{hours-per-week}\in[50,55)\:\wedge$\\$\text{occupation}=\text{Exec-managerial}\:\wedge$\\$\text{sex}=\text{Male}$} & \makecell[r]{13.0000\\\ \\\ \\\ } & \makecell[r]{0.0357\\\ \\\ \\\ } & \makecell[r]{0.5357\\\ \\\ \\\ } & \makecell[r]{29\\\ \\\ \\\ } & \makecell[r]{0.9655\\\ \\\ \\\ } \\
\addlinespace[5pt] \makecell[r]{0.9371\\\ \\\ \\\ } & \makecell[l]{$\text{education}=\text{Masters}\:\wedge$\\$\text{hours-per-week}\in[50,55)\:\wedge$\\$\text{occupation}=\text{Exec-managerial}\:\wedge$\\$\text{sex}=\text{Male}$} & \makecell[r]{13.0000\\\ \\\ \\\ } & \makecell[r]{0.0357\\\ \\\ \\\ } & \makecell[r]{0.5357\\\ \\\ \\\ } & \makecell[r]{29\\\ \\\ \\\ } & \makecell[r]{0.9655\\\ \\\ \\\ } \\
\addlinespace[5pt] \makecell[r]{0.9311\\\ \\\ \\\ } & \makecell[l]{$\text{education}=\text{Masters}\:\wedge$\\$\text{hours-per-week}\in[50,55)\:\wedge$\\$\text{marital-status}=\text{Married-civ-spouse}\:\wedge$\\$\text{occupation}=\text{Exec-managerial}$} & \makecell[r]{11.0000\\\ \\\ \\\ } & \makecell[r]{0.0417\\\ \\\ \\\ } & \makecell[r]{0.5600\\\ \\\ \\\ } & \makecell[r]{26\\\ \\\ \\\ } & \makecell[r]{0.9615\\\ \\\ \\\ } \\
\midrule
\addlinespace[5pt] \makecell[r]{0.0000} & \makecell[l]{$\emptyset$} & \makecell[r]{213.0197} & \makecell[r]{0.9728} & \makecell[r]{0.9241} & \makecell[r]{11305} & \makecell[r]{0.2482} \\
\bottomrule
\end{tabular}

%% file: table_case_study_filtered_result_set_table_1_1_average_ranking_loss_True.tex
\begin{tabular}{@{}r@{\hspace{2mm}}l@{\hspace{2mm}}r@{\hspace{2mm}}r@{\hspace{2mm}}r@{\hspace{2mm}}r@{\hspace{2mm}}r@{}}
\toprule
Interestingness & Pattern & ARL & PR AUC & ROC AUC & Cover & NCR \\
\midrule
\addlinespace[5pt] \makecell[r]{609186.2021} & \makecell[l]{$\text{marital-status}=\text{Married-civ-spouse}$} & \makecell[r]{354.2510} & \makecell[r]{0.8660} & \makecell[r]{0.8508} & \makecell[r]{5280} & \makecell[r]{0.4496} \\
\addlinespace[5pt] \makecell[r]{389481.0168} & \makecell[l]{$\text{relationship}=\text{Husband}$} & \makecell[r]{315.8451} & \makecell[r]{0.8619} & \makecell[r]{0.8485} & \makecell[r]{4638} & \makecell[r]{0.4495} \\
\addlinespace[5pt] \makecell[r]{34597.0291} & \makecell[l]{$\text{sex}=\text{Male}$} & \makecell[r]{223.1345} & \makecell[r]{0.9541} & \makecell[r]{0.9053} & \makecell[r]{7581} & \makecell[r]{0.3109} \\
\addlinespace[5pt] \makecell[r]{9394.5568} & \makecell[l]{$\text{capital-gain}\in[0.0,114.0)$} & \makecell[r]{216.3884} & \makecell[r]{0.9711} & \makecell[r]{0.9015} & \makecell[r]{10331} & \makecell[r]{0.2126} \\
\addlinespace[5pt] \makecell[r]{4270.2607\\\ } & \makecell[l]{$\text{capital-gain}\in[0.0,114.0)\:\wedge$\\$\text{sex}=\text{Male}$} & \makecell[r]{228.3650\\\ } & \makecell[r]{0.9505\\\ } & \makecell[r]{0.8764\\\ } & \makecell[r]{6821\\\ } & \makecell[r]{0.2708\\\ } \\
\midrule
\addlinespace[5pt] \makecell[r]{0.0000} & \makecell[l]{$\emptyset$} & \makecell[r]{213.0197} & \makecell[r]{0.9728} & \makecell[r]{0.9241} & \makecell[r]{11305} & \makecell[r]{0.2482} \\
\bottomrule
\end{tabular}

%% file: table_case_study_filtered_result_set_table_1_1_sklearn.metrics.roc_auc_score_True.tex
\begin{tabular}{@{}r@{\hspace{2mm}}l@{\hspace{2mm}}r@{\hspace{2mm}}r@{\hspace{2mm}}r@{\hspace{2mm}}r@{\hspace{2mm}}r@{}}
\toprule
Interestingness & Pattern & ARL & PR AUC & ROC AUC & Cover & NCR \\
\midrule
\addlinespace[5pt] \makecell[r]{316.1949} & \makecell[l]{$\text{marital-status}=\text{Married-civ-spouse}$} & \makecell[r]{354.2510} & \makecell[r]{0.8660} & \makecell[r]{0.8508} & \makecell[r]{5280} & \makecell[r]{0.4496} \\
\addlinespace[5pt] \makecell[r]{286.2383} & \makecell[l]{$\text{relationship}=\text{Husband}$} & \makecell[r]{315.8451} & \makecell[r]{0.8619} & \makecell[r]{0.8485} & \makecell[r]{4638} & \makecell[r]{0.4495} \\
\addlinespace[5pt] \makecell[r]{67.9166\\\ } & \makecell[l]{$\text{capital-gain}\in[0.0,114.0)\:\wedge$\\$\text{marital-status}=\text{Married-civ-spouse}$} & \makecell[r]{370.6283\\\ } & \makecell[r]{0.8473\\\ } & \makecell[r]{0.8053\\\ } & \makecell[r]{4628\\\ } & \makecell[r]{0.4114\\\ } \\
\addlinespace[5pt] \makecell[r]{64.1438} & \makecell[l]{$\text{sex}=\text{Male}$} & \makecell[r]{223.1345} & \makecell[r]{0.9541} & \makecell[r]{0.9053} & \makecell[r]{7581} & \makecell[r]{0.3109} \\
\addlinespace[5pt] \makecell[r]{63.0874} & \makecell[l]{$\text{capital-gain}\in[0.0,114.0)$} & \makecell[r]{216.3884} & \makecell[r]{0.9711} & \makecell[r]{0.9015} & \makecell[r]{10331} & \makecell[r]{0.2126} \\
\midrule
\addlinespace[5pt] \makecell[r]{0.0000} & \makecell[l]{$\emptyset$} & \makecell[r]{213.0197} & \makecell[r]{0.9728} & \makecell[r]{0.9241} & \makecell[r]{11305} & \makecell[r]{0.2482} \\
\bottomrule
\end{tabular}

%% file: table_case_study_filtered_result_set_table_1_1_prc_auc_score_True.tex
\begin{tabular}{@{}r@{\hspace{2mm}}l@{\hspace{2mm}}r@{\hspace{2mm}}r@{\hspace{2mm}}r@{\hspace{2mm}}r@{\hspace{2mm}}r@{}}
\toprule
Interestingness & Pattern & ARL & PR AUC & ROC AUC & Cover & NCR \\
\midrule
\addlinespace[5pt] \makecell[r]{460.8622} & \makecell[l]{$\text{marital-status}=\text{Married-civ-spouse}$} & \makecell[r]{354.2510} & \makecell[r]{0.8660} & \makecell[r]{0.8508} & \makecell[r]{5280} & \makecell[r]{0.4496} \\
\addlinespace[5pt] \makecell[r]{420.2511} & \makecell[l]{$\text{relationship}=\text{Husband}$} & \makecell[r]{315.8451} & \makecell[r]{0.8619} & \makecell[r]{0.8485} & \makecell[r]{4638} & \makecell[r]{0.4495} \\
\addlinespace[5pt] \makecell[r]{87.2251} & \makecell[l]{$\text{occupation}=\text{Exec-managerial}$} & \makecell[r]{71.7942} & \makecell[r]{0.9121} & \makecell[r]{0.9027} & \makecell[r]{1518} & \makecell[r]{0.4862} \\
\addlinespace[5pt] \makecell[r]{79.5800} & \makecell[l]{$\text{occupation}=\text{Prof-specialty}$} & \makecell[r]{69.6873} & \makecell[r]{0.9104} & \makecell[r]{0.8978} & \makecell[r]{1467} & \makecell[r]{0.4649} \\
\addlinespace[5pt] \makecell[r]{74.2654} & \makecell[l]{$\text{hours-per-week}\in[50,55)$} & \makecell[r]{58.9793} & \makecell[r]{0.8939} & \makecell[r]{0.8844} & \makecell[r]{1113} & \makecell[r]{0.4582} \\
\midrule
\addlinespace[5pt] \makecell[r]{0.0000} & \makecell[l]{$\emptyset$} & \makecell[r]{213.0197} & \makecell[r]{0.9728} & \makecell[r]{0.9241} & \makecell[r]{11305} & \makecell[r]{0.2482} \\
\bottomrule
\end{tabular}

%% file: figure_injection_experiment_combined_arl.tex
\begin{figure}[H]
    \centering
    \begin{subfigure}[]{0.3\textwidth}
        \includegraphics[]{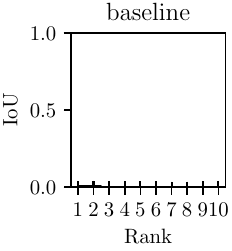}
    \end{subfigure}
    \begin{subfigure}[]{0.224743\textwidth}
        \includegraphics[]{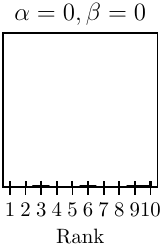}
    \end{subfigure}
    \begin{subfigure}[]{0.224743\textwidth}
        \includegraphics[]{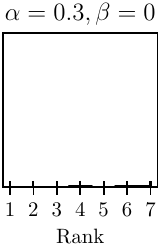}
    \end{subfigure}
    \begin{subfigure}[]{0.224743\textwidth}
        \includegraphics[]{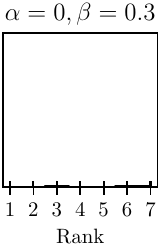}
    \end{subfigure}
    \caption{
    \gls{iou} of an injected subgroup with subgroups in a top-10 ARL result set, ordered decreasingly by interestingness.
    \enquote{baseline} was obtained searching in a base SubROC setting, the other results were obtained using the full SubROC framework with corresponding cover size weighting $\alpha$ and class balance weighting $\beta$.
    }
    \label{fig:subgroup_injection_result_arl}
\end{figure}

%% file: figure_injection_experiment_combined_pr_auc.tex
\begin{figure}[H]
    \centering
    \begin{subfigure}[]{0.3\textwidth}
        \includegraphics[]{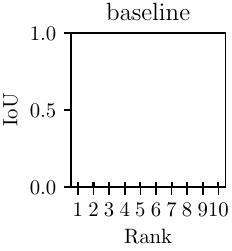}
    \end{subfigure}
    \begin{subfigure}[]{0.224743\textwidth}
        \includegraphics[]{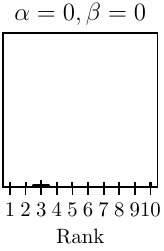}
    \end{subfigure}
    \begin{subfigure}[]{0.224743\textwidth}
        \includegraphics[]{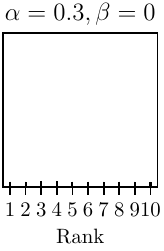}
    \end{subfigure}
    \begin{subfigure}[]{0.224743\textwidth}
        \includegraphics[]{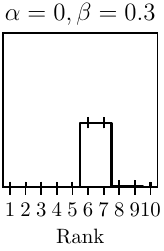}
    \end{subfigure}
    \caption{
    \gls{iou} of an injected subgroup with subgroups in a top-10 \gls{pr} \gls{auc} result set, ordered decreasingly by interestingness.
    \enquote{baseline} was obtained searching in a base SubROC setting, the other results were obtained using the full SubROC framework with corresponding cover size weighting $\alpha$ and class balance weighting $\beta$.
    }
    \label{fig:subgroup_injection_result_pr_auc}
\end{figure}

%% file: table_optimistic_estimates_depth_4_time_median_speedup_table_suppl_split.tex
\begin{tabular}{lr?r|r|r}
\rotatebox{90}{\makecell[l]{Performance\\Measure}}
 & \rotatebox{90}{\makecell[l]{Weight ($\alpha=\beta$)}}
 & \rotatebox{90}{\makecell[l]{\textbf{Mushroom}\\n/m=8124/22}} 
 & \rotatebox{90}{\makecell[l]{\textbf{Statlog}\\n/m=1000/20}} 
 & \rotatebox{90}{\makecell[l]{\textbf{Cancer}\\n/m=699/9}} \\

\midrule
\multirow[c]{4}{*}{ARL} & 0 & 617.8 & 132.4 & 5.0 \\
 & 0.1 & 410.1 & 85.5 & 2.8 \\
 & 0.3 & 405.6 & 84.2 & 2.6 \\
 & 1 & 403.0 & 85.6 & 2.7 \\
\cline{1-5}
\multirow[c]{4}{*}{PR AUC} & 0 & 699.8 & 84.4 & 4.8 \\
 & 0.1 & 211.9 & 62.3 & 2.9 \\
 & 0.3 & 89.7 & 85.9 & 2.9 \\
 & 1 & 78.6 & 85.7 & 2.9 \\
\cline{1-5}
\multirow[c]{4}{*}{ROC AUC} & 0 & 337.7 & 87.2 & 8.3 \\
 & 0.1 & 268.6 & 54.8 & 5.0 \\
 & 0.3 & 277.1 & 53.4 & 4.3 \\
 & 1 & 249.6 & 68.9 & 4.3 \\
\end{tabular}

%% file: table_optimistic_estimates_depth_4_time_median_table_average_ranking_loss.tex
\begin{tabular}{lrrrr}
\toprule
Dataset & 0 & 0.1 & 0.3 & 1 \\
\midrule
UCI Census-Income (KDD) & 114771.7 & 115310.2 & 114248.1 & 114822.9 \\
OpenML Adult & 1058.0 & 1056.0 & 1044.8 & 1061.8 \\
UCI Bank Marketing & 127.8 & 128.6 & 128.7 & 128.9 \\
UCI Credit Card Clients & 4826.9 & 4803.8 & 4867.5 & 4800.9 \\
UCI Mushroom & 119.2 & 119.4 & 118.5 & 119.1 \\
Statlog (German Credit Data) & 10.6 & 10.7 & 10.6 & 10.8 \\
UCI Breast Cancer Wisconsin & 0.1 & 0.1 & 0.1 & 0.1 \\
UCI Credit Approval & 1.9 & 1.8 & 1.8 & 1.9 \\
\bottomrule
\end{tabular}

%% file: table_optimistic_estimates_depth_4_time_median_table_roc_auc_score.tex
\begin{tabular}{lrrrr}
\toprule
Dataset & 0 & 0.1 & 0.3 & 1 \\
\midrule
UCI Census-Income (KDD) & 84443.9 & 85060.0 & 85223.8 & 84915.9 \\
OpenML Adult & 1055.2 & 1048.9 & 1036.1 & 1051.5 \\
UCI Bank Marketing & 160.0 & 159.6 & 160.2 & 161.4 \\
UCI Credit Card Clients & 5304.0 & 5305.0 & 5293.7 & 5275.4 \\
UCI Mushroom & 118.4 & 117.7 & 117.4 & 118.9 \\
Statlog (German Credit Data) & 16.2 & 16.5 & 16.4 & 16.5 \\
UCI Breast Cancer Wisconsin & 0.3 & 0.3 & 0.3 & 0.3 \\
UCI Credit Approval & 2.5 & 2.4 & 2.4 & 2.4 \\
\bottomrule
\end{tabular}

%% file: table_optimistic_estimates_depth_4_time_median_table_prc_auc_score.tex
\begin{tabular}{lrrrr}
\toprule
Dataset & 0 & 0.1 & 0.3 & 1 \\
\midrule
UCI Census-Income (KDD) & 118601.6 & 118895.8 & 118370.9 & 118797.9 \\
OpenML Adult & 1123.3 & 1134.3 & 1127.3 & 1124.1 \\
UCI Bank Marketing & 145.4 & 146.4 & 147.7 & 146.9 \\
UCI Credit Card Clients & 5139.3 & 5138.0 & 5117.3 & 5122.6 \\
UCI Mushroom & 161.0 & 162.8 & 162.8 & 162.6 \\
Statlog (German Credit Data) & 14.2 & 14.8 & 14.2 & 14.1 \\
UCI Breast Cancer Wisconsin & 0.2 & 0.2 & 0.2 & 0.2 \\
UCI Credit Approval & 2.2 & 2.2 & 2.2 & 2.2 \\
\bottomrule
\end{tabular}

%% file: table_optimistic_estimates_depth_4_num_visited_speedup_table.tex
\begin{tabular}{lrrrrrrrrr}
\toprule
 &  & \rotatebox{90}{\makecell{UCI Census-Income\\(KDD)}} & \rotatebox{90}{\makecell{OpenML Adult}} & \rotatebox{90}{\makecell{UCI Bank\\Marketing}} & \rotatebox{90}{\makecell{UCI Credit\\Card Clients}} & \rotatebox{90}{\makecell{UCI Mushroom}} & \rotatebox{90}{\makecell{Statlog (German\\Credit Data)}} & \rotatebox{90}{\makecell{UCI Breast\\Cancer Wisconsin}} & \rotatebox{90}{\makecell{UCI Credit Approval}} \\
Performance\\Measure & Weight ($\alpha=\beta$) &  &  &  &  &  &  &  &  \\
\midrule
\multirow[c]{4}{*}{ARL} & 0 & 6.62\% & 3.82\% & 0.02\% & 4.55\% & 0.03\% & 0.13\% & 6.68\% & 0.48\% \\
 & 0.1 & 6.07\% & 3.48\% & 0.02\% & 4.55\% & 0.03\% & 0.13\% & 6.68\% & 0.48\% \\
 & 0.3 & 5.20\% & 2.82\% & 0.02\% & 4.55\% & 0.03\% & 0.13\% & 6.68\% & 0.48\% \\
 & 1 & 3.21\% & 1.07\% & 0.02\% & 4.55\% & 0.03\% & 0.13\% & 6.68\% & 0.48\% \\
\cline{1-10}
\multirow[c]{4}{*}{PR AUC} & 0 & 42.47\% & 47.13\% & 0.02\% & 17.65\% & 0.03\% & 0.13\% & 6.68\% & 0.48\% \\
 & 0.1 & 42.47\% & 60.95\% & 0.02\% & 42.01\% & 0.07\% & 0.13\% & 6.68\% & 0.48\% \\
 & 0.3 & 42.47\% & 63.18\% & 0.02\% & 42.77\% & 0.20\% & 0.13\% & 6.68\% & 0.48\% \\
 & 1 & 19.63\% & 7.85\% & 0.02\% & 25.62\% & 0.25\% & 0.13\% & 6.68\% & 0.48\% \\
\cline{1-10}
\multirow[c]{4}{*}{ROC AUC} & 0 & 61.26\% & 73.32\% & 0.02\% & 42.99\% & 0.04\% & 0.13\% & 7.15\% & 0.48\% \\
 & 0.1 & 61.26\% & 73.32\% & 0.02\% & 42.99\% & 0.04\% & 0.13\% & 7.15\% & 0.48\% \\
 & 0.3 & 61.29\% & 73.32\% & 0.02\% & 42.99\% & 0.04\% & 0.13\% & 7.15\% & 0.48\% \\
 & 1 & 26.83\% & 9.89\% & 0.02\% & 36.11\% & 0.04\% & 0.13\% & 7.15\% & 0.48\% \\
\cline{1-10}
\bottomrule
\end{tabular}